\documentclass{article} 
\usepackage{iclr2025_conference,times}
\iclrfinalcopy 

\usepackage{amsmath,amsfonts,bm}

\def\eqref#1{equation~\ref{#1}}

\def\1{\bm{1}}

\DeclareMathAlphabet{\mathsfit}{\encodingdefault}{\sfdefault}{m}{sl}
\SetMathAlphabet{\mathsfit}{bold}{\encodingdefault}{\sfdefault}{bx}{n}

\usepackage{hyperref}
\usepackage{url}

\usepackage[utf8]{inputenc} 
\usepackage[T1]{fontenc}    
\usepackage{hyperref}       
\usepackage{url}            
\usepackage{booktabs}       
\usepackage{amsfonts}       
\usepackage{nicefrac}       
\usepackage{microtype}      
\usepackage{xcolor}         

\usepackage{amsmath}
\usepackage{amsthm}
\usepackage{amssymb}
\usepackage{bbm}
\usepackage{tikz}
\usetikzlibrary{automata,positioning,arrows,shapes.geometric,arrows.meta}
\usepackage{algorithm}
\usepackage{algpseudocode}
\newtheorem{definition}{Definition}
\usepackage{graphicx}
\usepackage{subcaption}
\newtheorem{theorem}{Theorem}
\usepackage{natbib}
\usepackage{tabularx} 
\usepackage{array}    
\usepackage{graphicx} 
\usepackage{booktabs, multirow, array, caption, colortbl, xcolor, threeparttable}
\usepackage{siunitx}
\newcommand{\up}{$\uparrow$}
\newcommand{\down}{$\downarrow$}
\title{Global-Order GFlowNets}

\author{%
  Lluís Pastor Pérez \thanks{Work done in an internship at Sony Europe B.V. Main author.} \\
  Department of Computer Science\\
  ETH Zürich\\
  \texttt{lpastor@ethz.ch } \\
 \And
  Javier Alonso García \\
  Sony Europe B.V. \\
  Stuttgart, Germany \\
  \texttt{javier.alonso@sony.com} \\
  \And
  Lukas Mauch \\
  Sony Europe B.V. \\
  Stuttgart, Germany \\
  \texttt{lukas.mauch@sony.com} \\
}

\begin{document}

\maketitle

\begin{abstract}
Order-Preserving (OP) GFlowNets have demonstrated remarkable success in tackling complex multi-objective (MOO) black-box optimization problems using stochastic optimization techniques. Specifically, they can be trained online to efficiently sample diverse candidates near the Pareto front. A key advantage of OP GFlowNets is their ability to impose a local order on training samples based on Pareto dominance, eliminating the need for scalarization – a common requirement in other approaches like Preference-Conditional GFlowNets. However, we identify an important limitation of OP GFlowNets: imposing this local order on training samples can lead to conflicting optimization objectives. To address this issue, we introduce Global-Order GFlowNets, which transform the local order into a global one, thereby resolving these conflicts. Our experimental evaluations on various benchmarks demonstrate a small but consistent performance improvement.
\end{abstract}

\section{Introduction}
Multi-objective optimization is a complex problem that arises in many fields, where multiple competing objectives must be optimized simultaneously. One of the biggest challenges lies in the black-box scenario, where the objectives can only be evaluated indirectly, without explicit knowledge of their underlying functions. In such cases, we often rely on Bayesian or stochastic optimization methods to navigate the solution space and to sample 
diverse candidates near the Pareto front \cite{MOO_basics}.

Recently, GFlowNets have been introduced to solve challenging black-box optimization methods with great success. Numerous methods have been developed to adapt GFlowNets to MOO problems, addressing the challenge of multi-dimensional rewards and Pareto dominance. In these problems, rewards are often conflicting, meaning that improving one objective may worsen another. For example, in Neural Architecture Search (NAS), higher training accuracy often comes with a larger number of parameters, leading to increased latency and computational cost.

One such approach are OP GFlowNets (OP-GFNs) \cite{OP-Gfns}, which use new reward mechanism based on Pareto dominance across subsets. In this approach, the probability of choosing a sample is either uniformly distributed across the Pareto set or zero otherwise. The learned reward is designed to preserve the ordering of utility, minimizing the Kullback-Leibler (KL) divergence \cite{KL} between the true Pareto distribution and the learned distribution. This ensures that the generated samples respect the Pareto dominance structure.

We identified one major shortcoming of OP-GFNs, namely: "The way how OP GFNs impose the local order on the training samples can lead to conflicting training objectives." 

We summarize the contributions of this paper: 1) We propose two new approaches to convert the local order into a global one, defining a new training paradigm for GFlowNets, namely, Global-Order GFlowNets. These methods differ in computational complexity and offer distinct orders, yet both are consistent with the local relation of Pareto dominance. 2) Using a simple example, we demonstrate how the restrictive local order used in OP-GFNs can lead to inconsistencies. 3) We test our methods on different benchmarks, achieving comparable results and, in some cases, improving the performance.

\section{Related Work}
\subsection{Multi-Objective Optimization}
Multi-Objective Optimization (MOO), also known as Pareto optimization, deals with problems involving more than one objective function to be optimized simultaneously. Without loss of generality, we consider the problem of maximizing a function. Formally, a MOO problem is defined as follows:

\begin{definition}[MOO problem]
Let $\mathcal{X}$ be the data (solution) space, $\mathcal{Y}$ the d-dimensional objective space and $F = (f_1, f_2, ..., f_d) : \mathcal{X} \to \mathcal{Y}$. Then, the MOO problem associated with this setup is
    $$\mathrm{argmax}_{x\in \mathcal{X}} (f_1(x),f_2(x),...,f_d(x))$$
\end{definition}

The essence of MOO is to optimize multiple objectives simultaneously, with no objective being more important than another. Unlike single-objective optimization, MOO aims to find a set of optimal solutions, known as the Pareto front, which represents trade-offs among the objectives.

\begin{definition}[Pareto Dominance and Pareto Set]
Let $\mathcal{X}$, $\mathcal{Y}$, and $f_1, f_2, \ldots, f_d$ be as defined above. For $x, x' \in \mathcal{X}$, $x$ Pareto dominates (or simply dominates) $x'$ if:
\begin{enumerate}
\item For every $i \in  \{1, \ldots, d\}$, $f_i(x) \geq f_i(x')$.
\item There is a $j \in \{1, \ldots, d\}$ such that $f_j(x) > f_j(x')$.
\end{enumerate}
An $x \in \mathcal{X}$ is non-dominated if $\nexists x' \in \mathcal{X}$ such that $x'$ dominates $x$. The set $\mathcal{P}_{\mathcal{X}}$, consisting of non-dominated samples, is the Pareto set. The images under $f_1, ... f_d$ of all $x\in \mathcal{P}_{\mathcal{X}}$ is the Pareto front.
\end{definition}

We consider a Multi-Objective Optimization (MOO) in a black-box scenario, which means that the objective functions $f_i(x)$ are not known analytically and can only be evaluated through stochastic optimization. Our primary interest lies in sampling diverse candidates near the Pareto front with high probability. This requires efficient stochastic optimization methods to navigate the uncertainty inherent in this black-box setup.

\subsection{Performance measures for Multi-Objective Optimization}
MOO algorithms must address two critical aspects: 1) {\bf Fidelity}: Ensuring that sampled solutions are close to the true Pareto front, indicating a good approximation of the optimal trade-offs among the objectives. This is typically measured by distance-based metrics such as the \emph{Inverted Generational Distance (IGD\(^+\))} and the \emph{Averaged Hausdorff Distance (\(d_H\))} \cite{GD+,IGD}. 2) {\bf Coverage}: Guaranteeing that we cover a significant portion of the entire Pareto front, providing a comprehensive understanding of the optimal solution space. This is typically measured by diversity and spread metrics such as the \emph{Hypervolume Indicator (HV)}, \emph{Pareto Coverage}, and \emph{Pareto-Clusters Entropy (PC-ent)} \cite{HV,goal_cond_gflownets}. We employ more metrics to compare the different methods, with details available in Appendix \ref{sec:Metrics}.

\subsection{GFlowNets for multi-objective optimization}
Generative Flow Networks (GFlowNets) \cite{gflownet,gflownet_foundations} have emerged as a novel class of probabilistic generative models. GFlowNets are controllable, in the sense that they can generate samples proportional to a given reward function $x \propto R(x)$. Compared to MCMC methods, they amortize the sampling in a training step and, hence, are more efficient \cite{mcmcforml}. 

More specifically, GFlowNets decompose the sampling process into sequences of actions that are applied to, and that modify, an object $s_0$, forming a trajectory of partial objects $\tau=(s_0, s_1, ..., s_n)$ that terminate in an observation $x=s_n$. GFlowNets are typically parametrized by a triplet $P_F(s_t|s_{t-1};\theta)$, $P_B(s_{t-1}|s_{t};\theta)$ and $Z_\theta$, namely the forward policy, the backward policy and the partition function. By forcing $P_F$, $P_B$ and $Z_{\theta}$ into balance over a set of training trajectories, i.e. by enforcing for each trajectory that
\begin{equation}
    R(x)\prod_{t=1}^nP_B(s_{t-1}|s_t;\theta) = Z_\theta \prod_{t=1}^nP_F(s_t|s_{t-1};\theta),
    \label{eqn:tb_objective_one}
\end{equation}
we assure that objects $x$ are generated from $s_0$ with $P(x;\theta) = \frac{R(x)}{\sum_{\forall x' \in \mathcal{X}} R(x')}$, when adhering to the forward policy $P_F$.

Similar to MCMC methods, GFlowNets can be easily applied for stochastic optimization. Let $f(x)$ be an objective function we want to maximize, we can choose $R(x)=\exp(\beta f(x))$ to sample candidates that maximize $f(x)$ proportionally more often \cite{Kirkpatrick1983}. Here, $\beta \in \mathbb{R}^+$ is a temperature term that controls the variance.

There are different methods to adapt GFlowNets to MOO problems, One approach is Preference-Conditional GFlowNets (PC-GFNs) \cite{moo_gfn}, which use scalarization on the rewards. More specifically, scalarization means choosing preferences $w: \: R_{w}(x) = w^\top F(x)$, where $\mathbbm{1}^\top w=1, w_k\geq 0$.  The scalarized objective is then used as a reward signal for a conditional GFlowNet. Of course, the choice of $w$ strongly influences the shape of $R_w(x)$ and therefore the regions in which the GFlowNet will generate the most samples. Goal-Conditioned GFlowNets (GC-GFNs) \cite{goal_cond_gflownets} tried to control the generation and introduced different focus regions that can steer the generation process. 

The Order-Preserving (OP) GFlowNet  \cite{OP-Gfns} is a variant that is particularly useful for MOO. It does not rely on a predefined reward signal $R(x)$ as it introduces a new reward mechanism based on Pareto dominance across subsets of $\mathcal{X}$. More specifically, let $\mathcal{X}' \subset \mathcal{X}$, an OP GFlowNet defines the probability of a sample $x \in \mathcal{X}'$ to be in the Pareto set $\mathcal{P}_{\mathcal{X}'}$ of the given subset $\mathcal{X}'$ as
\begin{align}\label{eq:general_formulation}
    P(x|\mathcal{X}')= \frac{ \mathbf{1}[x\in \mathcal{P}_{\mathcal{X}'}]}{|\mathcal{P}_{\mathcal{X}'}|},
\end{align}
where $\mathbf{1}[\cdot]$ is the indicator function. 

An OP GFlowNet is trained to approximate $P(x|\mathcal{X}')$ for any subset $\mathcal{X}'$, i.e., to sample uniformly from the Pareto set of any subset $\mathcal{X}'$. Let $R(x;\theta)$ be the reward that is implicitly defined by Eq.\ref{eqn:tb_objective_one} for given $P_F$, $P_B$ and $Z_\theta$, OP GFlowNets minimize the Kullback-Leibler divergence
\begin{align}
\label{eq:main_order_loss}
    \mathcal{L}_{\rm OP}(\mathcal{X}') = \text{KL}( P(x|\mathcal{X}')\|P(x|\mathcal{X}';\theta)),
\end{align}
across the subsets $\mathcal{X}' \subset \mathcal{X}$, where the conditional $P(x|\mathcal{X}';\theta) = \frac{R(x;\theta)}{\sum_{x'\in \mathcal{X}'}R(x';\theta)}, \: \forall x\in \mathcal{X}'$. The OP GFlowNet then learns a distribution $P(x;\theta)$ over the full set $\mathcal{X}$, that is consistent with all marginals, i.e.
\begin{align}
    P(x;\theta) = P(x|\mathcal{X}'; \theta)P(\mathcal{X}'; \theta), \: \forall \mathcal{X}' \subseteq \mathcal{X}.
\end{align}

Note, that for all subsets $\mathcal{X}': \: \mathcal{P}_{\mathcal{X}} \cap \mathcal{X}' = \emptyset \quad \Rightarrow \quad P(\mathcal{X}';\theta) = 0$ such that the conditionals $P(x|\mathcal{X}'; \theta)$ are consistent with $P(x; \theta)$. For scalar optimization problems with a single optimum, a possible $P(x; \theta)$ with the required properties exists and is the limit of the softmax $P(x; \theta) =\lim_{\gamma \rightarrow \infty} \frac{e^{\gamma g(f(x))}}{\sum_{x' \in \mathcal{X}} e^{\gamma g(f(x'))}}$, where $g(\cdot)$ is a monotonically increasing function that preserves the order on the image of $f$, that is, $f(x_1) \leq f(x_2) \Rightarrow g(f(x_1)) \leq g(f(x_2))$.

OP-GFNs' main advantage is that their loss function does not need the definition of a scalar reward. They are trained to sample proportional to given target distributions in the subset $\mathcal{X}'$ of the solution space. These target distributions are uniform over the Pareto set of these subsets, and $0$ otherwise.

\section{Method}

\subsection{Local Order Dilemma of OP GFNs}
We noticed that different to the scalar case, training with the target distribution across subsets that has been proposed for OP GFNs can lead to inconsistencies that cannot be resolved.

Figure~\ref{fig:local_dilema} shows a simple failure case, where we want to maximize two objective functions $F=(f_1, f_2)$ over the set $\mathcal{X} = \{x_1,x_2,x_3,x_4\}$. Ideally, we want to learn a $P(x;\theta)$ that is uniform over the Pareto set $\mathcal{P}_{\mathcal{X}} = \{x_3, x_4\}$ and zero otherwise, by fitting the conditionals over the subsets $\mathcal{X}_1 = \{x_2, x_4\}$, $\mathcal{X}_2 =\{x_2, x_3\}$, $\mathcal{X}_3 = \{x_3, x_4\}$, using the target distributions from Eq.~\ref{eq:general_formulation}. 

However, for this specific example we cannot construct a $P(x)$ that is consistent with all possible conditional reference distributions. More specifically, we recognize that each subset contains at least one Pareto optimal point, meaning that none of the marginals $P(\mathcal{X}_k), \: k=1,2,3$ are allowed to vanish. Further, Eq.~\ref{eq:general_formulation} imposes the conditions: 1) $P(x_2)=P(x_4)=0.5$, 2) $P(x_3)=P(x_4)=0.5$ and 3) $P(x_2|\mathcal{X}_2)=0$. The last condition leads to a contradiction, because $P(\mathcal{X}_2) \neq 0$, yielding $P(x_2)=0$.

This problem arises, because we do not allow the probability of a point within the Pareto set of 
$\mathcal{X}_k$ to vanish. However, this is necessary in order to resolve this contradiction. In the following, we resolve this issue by: 1) Defining a global order over the elements of $\mathcal{X}$ that is consistent with the Pareto dominance. 2) Training the GFlowNet to sample proportional to this global order.

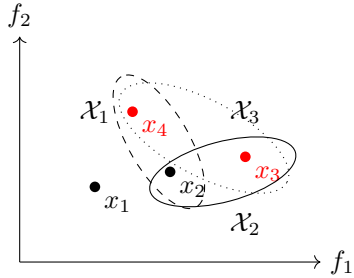
\begin{figure}[h!]
    \centering
        \begin{tikzpicture}[scale=1]
  \draw[->] (0,0) -- (4,0) node[right] {$f_1$};
  \draw[->] (0,0) -- (0,3) node[above] {$f_2$};

  \fill (1,1) circle (2pt) node[below right] {$x_1$};
  \fill (2,1.2) circle (2pt) node[below right] {$x_2$};
  \fill[red] (3,1.4) circle (2pt) node[below right] {$x_3$};
  \fill[red] (1.5,2) circle (2pt) node[below right] {$x_4$};

  \draw[dashed, rotate around={120:(1.85,1.6)}] (1.85,1.6) ellipse (1 and 0.4);
  \node at (1, 2) {$\mathcal{X}_1$};
  \draw[dotted, rotate around={-28:(2.45,1.65)}] (2.45,1.65) ellipse (1.25 and 0.5);
  \node at (3, 2) {$\mathcal{X}_3$};

  \draw[solid, rotate around={15:(2.7,1.2)}] (2.7,1.2) ellipse (1 and 0.4);
  \node at (3, 0.5) {$\mathcal{X}_2$};
\end{tikzpicture}
    \caption{An example for the local order dilemma of OP GFNs. We can construct three different subsets $\mathcal{X}_1$, $\mathcal{X}_2$, $\mathcal{X}_3$, for which we can no longer construct $P(x;\theta)$ that is consistent with all conditionals.}
    \label{fig:local_dilema}
\end{figure}

\subsection{Global Orders}
This contradiction proves the need for a global ordering perspective. By assigning a global $\hat{R}$ that respects Pareto's dominance universally, rather than subset-by-subset basis, we completely avoid such inconsistencies.
Our goal is to define a global ordering function $\hat{R}$ that assigns a rank or score to each point $x \in X$, consistently reflecting Pareto dominance, so that it does not depend on any subset, but on the seen data $\mathcal{X}$. Specifically, if $x_1$ dominates $x_2$, then we must have $\hat{R}(x_1) > \hat{R}(x_2)$. Contrarily to OP-GFNs, if $x_1$ and $x_2$ are not comparable, it is up to the global order to decide whether $\hat{R}(x_1)$ is greater, equal, or lower than $\hat{R}(x_2)$. By introducing a global perspective, we avoid the contradictions that arise when partial orders are combined from multiple subsets. We propose two main methods for defining such global order, and we name Global-Order GFlowNets to GFlowNets with a global ordering.

\subsubsection{Global Rank}
\label{sec:global_rank}

The Global Rank method (Algorithm~\ref{algo:global_rank}) iteratively identifies Pareto fronts within the dataset $D$, assigning integer ranks that show how close each point is to the Pareto front. We first find the Pareto front of $D$ and assign it the lowest temporary rank. After removing these points, we find the next Pareto front from the remaining set and assign it the next integer rank, and so on, until all points have been assigned a rank (or earlier if we decide to stop the process). In the end, we invert the ranking so that points on the first Pareto front receive the highest $\hat{R}$ values.

\begin{algorithm}[t]
\caption{\textsc{Global Rank}}
\label{algo:global_rank}
\begin{algorithmic}[1]
\Require Dataset $D$, objective functions $f_1,\ldots,f_d$

\State $i \gets 0$
\State $D_{\text{orig}} \gets D$
\While {$|D| > 0$}
  \State $P_{\text{front}} \gets \texttt{Compute\_Pareto}(D, f_1, \ldots, f_d)$
  \For {$p \in P_{\text{front}}$}
    \State $rank(p) \gets i$ \Comment{Assign current rank to Pareto front points}
  \EndFor
  \State $D \gets D \setminus P_{\text{front}}$
  \State $i \gets i + 1$
\EndWhile
\State $max\_rank \gets i$ \Comment{Maximum assigned rank}
\For {$p \in D_{\text{orig}}$}
  \State $\hat{R}(p) \gets max\_rank - rank(p)$ \Comment{Invert ranks}
\EndFor

\State \Return $\hat{R}$ \Comment{Return the global ranks for all points in $D_{\text{orig}}$}
\end{algorithmic}
\end{algorithm}
Although Global Rank guarantees consistency with Pareto dominance, it can be computationally expensive since it repeatedly computes Pareto fronts. For large-scale problems, the number of iterations may be capped, assigning a default minimal rank to all remaining points after a cutoff. This slightly relaxes theoretical guarantees for very poor solutions, but provides substantial computational savings.

Once $\hat{R}$ is defined through global ranking, we can train the GFlowNet such that $R(X;\theta)$ approximates either $\hat{R}(X) = (\hat{R}(x_1), \ldots, \hat{R}(x_B))$, its softmax transformation, or the indicator function $\mathbf{1}_{u = \max(\hat{R}(X))}[\hat{R}(X)]$. The choice depends on whether we prioritize uniform exploration along Pareto fronts (for example, by using softmax) or place a stricter focus on top-ranked points.

\subsubsection{Nearest Neighbor Order}

The Nearest Neighbor Order (Algorithm ~\ref{algo:nn_rank}) is an alternative global ranking approach that bases a point’s score on its distance to the Pareto front. Points on the Pareto front receive $\hat{R}=0$, and all other points receive negative values proportional to their minimal distance to any Pareto optimal point.

We first compute the Pareto front $P$ of $D$. For each $x \notin P$, we measure the distance $d(x, P) = \min_{p \in P} d(x, p)$. We then set $\hat{R}(x) = -d(x,P)$, while $\hat{R}(p)=0$ for $p \in P$. Alternatively, the distance can also be computed by interpolating along the Pareto front.

\begin{algorithm}[t]
\caption{\textsc{Nearest Neighbor Order}}
\label{algo:nn_rank}
\begin{algorithmic}[1]
\Require Dataset $D$, objective functions $f_1,\ldots,f_d$, distance metric $d(\cdot,\cdot)$

\State $P \gets \texttt{Compute\_Pareto}(D, f_1, \ldots, f_d)$
\For{$p \in P$}
  \State $\hat{R}(p) \gets 0$ \Comment{Points on the Pareto front have $\hat{R}=0$}
\EndFor

\For{$x \in D \setminus P$}
  \State $distances \gets []$ \Comment{Collect distances to all Pareto front points}
  \For{$p \in P$}
    \State $distances.\text{append}(d(x, p))$
  \EndFor
  \State $\hat{R}(x) \gets -\min(distances)$ \Comment{Assign rank as negative min. distance to Pareto front}
\EndFor

\State \Return $\hat{R}$ \Comment{Return $\hat{R}$ assignments for all points in $D$}
\end{algorithmic}
\end{algorithm}

Normalization of reward scales is critical here to avoid bias in the distance computations. Additionally, this approach may be less suitable if rewards differ significantly in their scales or distributions.

We show an example on how these orders, although defined under the same principle, can result in largely different reward distributions. Consider the discretization of the square with (target) rewards \( r: [0, 1] \times [0, 1] \to \mathbb{R}^2 \), $
r(x, y) = \left( \pi x \cos(\pi x),\ \pi y \sin(\pi y) \right).$ After applying the two orders, we obtain completely different values of $\hat{R}$, as shown in Figure \ref{fig:rew_sin}.
\begin{figure}[h]
    
    \centering
    \begin{subfigure}[b]{0.45\textwidth}
        \includegraphics[width=\textwidth]{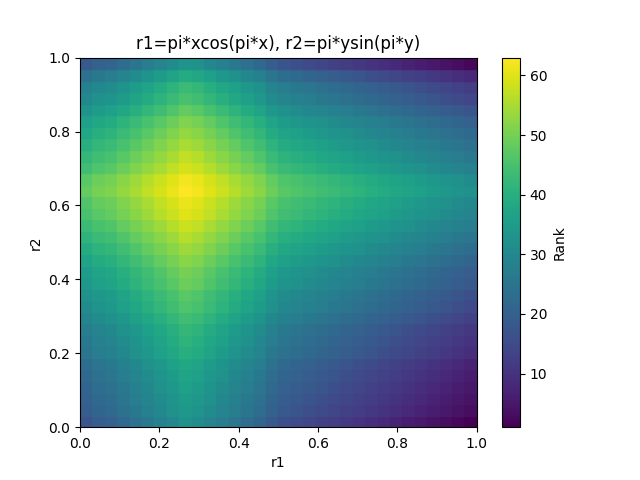}
        
    \end{subfigure}
     \begin{subfigure}[b]{0.45\textwidth}
        \includegraphics[width=\textwidth]{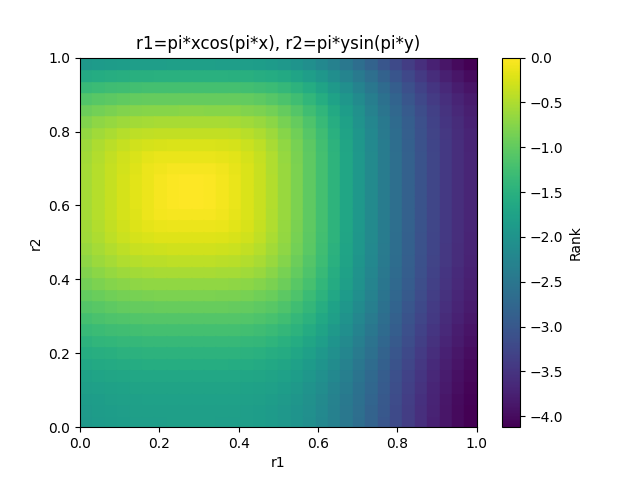}
    \end{subfigure}
\caption{$\hat{R}$ computed with Global Rank (left) and Nearest Neighbor (right) algorithms}
\label{fig:rew_sin}
\end{figure}

\section{Experiments}
In this section, we evaluate our different Global-Order GFlowNets across various benchmarks to demonstrate improvements over previous results. The definition of the evaluation metrics can be found in Appendix \ref{sec:Metrics}. We begin this section explaining the usage of replay buffers. Then, we present results for key benchmarks such as Fragment-based molecule generation, QM9, and DNA, detailing each benchmark's importance, methods, parameters, and comparisons with other algorithms. For the synthetic benchmarks, HyperGrid and N-Grams, we have included the results and detailed analysis in Appendix \ref{appendix:experiments}. Our methods are marked with an asterisk (*) in all tables.

\subsubsection{Replay Buffer and Cheap GFlowNets}
In MOO, traditional training methods often rely solely on the samples generated at each step, leading to unstable training and convergence issues. To address this, most experiments incorporate a replay buffer, where newly generated trajectories are stored and randomly sampled for training, stabilizing the learning process.

In single-objective scenarios, it is common to sample from both high- and low-reward experiences, while in MOO, samples are drawn from the Pareto front of all observed samples. We implement a warm-up phase to avoid the network getting stuck in a local optima.

The training process involves drawing new samples, updating the Pareto front, and drawing batches from both the buffer and the Pareto front, ensuring at least $k\in\mathbb{N}$ Pareto optimal samples are included. This approach allows for an efficient and cost-effective variant of Global Rank GFlowNets, named Cheap-GR-GFNs, by focusing only on the actual Pareto front, significantly reducing computational costs.

From now on, we refer to our methods as follows: Global-Rank GFlowNets (GR-GFNs), Trimmed Global-Rank GFlowNets with a maximum rank \( k \) (GR-GFNs (\( k \))), Cheap Global-Rank GFlowNets (Cheap GR-GFNs), Nearest Neighbor GFlowNets (NN-GFNs), and their variant with linear interpolation of the Pareto Front (NN-int-GFNs). Due to the amount of different experiments, most of the plots and the tables can be found in Appendix \ref{appendix:experiments}.
\subsection{DNA Sequence Generation}
GFlowNets function as samplers of trajectories, making them naturally suited for addressing various problems in biology and chemistry. In this task, we generate DNA sequences by adding one nucleobase at a time: adenine (A), cytosine (C), guanine (G), or thymine (T) \cite{dna, dna2, dna3, dna4}.
In this work, we evaluate \texttt{energy-pairs} and \texttt{energy-pins-pairs}, whose details can be found in Appendix \ref{appendix:dna}.
For consistency and due to the similarity of the problem structure, we compare our approach using the same algorithms (PC-GFNs and OP-GFNs) and parameters as in the N-Grams task found in Appendix \ref{appendix:ngrams}. Following the advice of \cite{moo_gfn, OP-Gfns}, we set $\beta = 80$ for PC-GFNs.
\subsubsection{Results}
The results are presented in Table \ref{tab:dna_results}. For the two rewards case we see that the three methods perform optimally, but for the three rewards, ours clearly outperform the other two, except for the Top-k diversity. In \cite{OP-Gfns}, authors claim that their method does not sample as close to the Pareto front as PC-GFNs, but their method explores more. This is repeated here with our method, now exploring more, yet getting a slightly worse Pareto front in \texttt{Energy-Pairs}, as shown in Figure \ref{fig:dna_plots}.

\begin{table}[h]
\caption{Results for DNA benchmark}
\centering
\sisetup{
  table-format=2.2,
  detect-weight=true,
  detect-inline-weight=math
}
\begin{tabular}{
  l
  S[table-format=1.2]
  S[table-format=2.2]
  S[table-format=1.2]
  S[table-format=1.2]
  S[table-format=2.2]
}
\toprule
\textbf{Energy - Pairs} & {HV (\up)} & {R$_2$ (\down)} & {PC-ent (\up)} & {d$_H(P',P)$ (\down)} & {Diversity (\up)} \\
\midrule
OP-GFNs           & \textbf{0.35} & \textbf{2.65} & \textbf{0.00} & \textbf{0.53} &3.98 \\
PC-GFNs           & \textbf{0.35} &\textbf{2.65} & \textbf{0.00} & \textbf{0.53} &5.27 \\
Cheap-GR-GFNs* & \textbf{0.35} & \textbf{2.65} & \textbf{0.00} & \textbf{0.53} &\textbf{\;\;6.97}\\
\midrule
\textbf{Energy - Pins - Pairs}  \\
\midrule
OP-GFNs           & 0.07 & 17.32 & 0.69 & 0.74 & \textbf{\;\;11.5} \\
PC-GFNs           & 0.08 & 17.08 & 0.69 & 0.75 & 6.34 \\
Cheap-GR-GFNs* & \textbf{0.11} & \textbf{14.44} & \textbf{1.39} & \textbf{0.67} & 6.49 \\
\bottomrule
\end{tabular}

\label{tab:dna_results}
\end{table}
\begin{figure}[h]
\centering
        \includegraphics[width=0.7\textwidth]{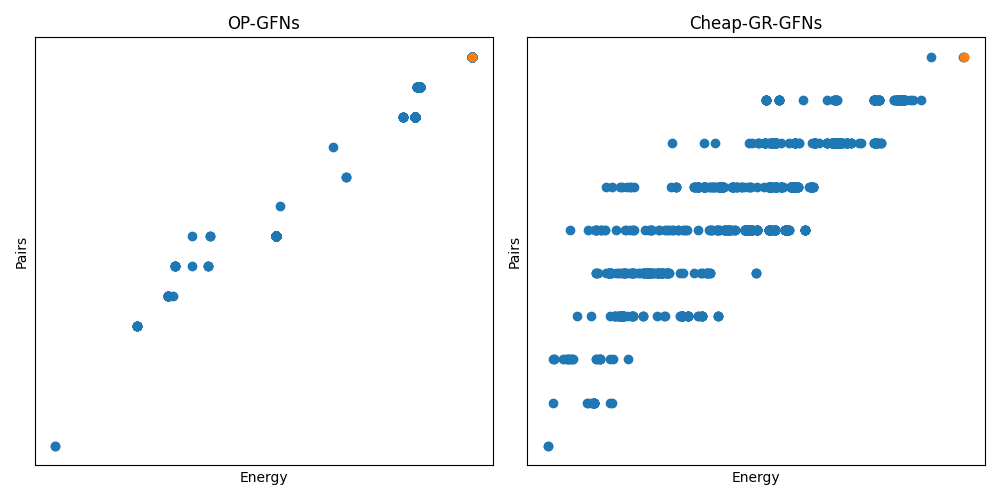}
    \caption{The 1280 generated samples, with 18 unique samples for OP-GFNs compared to 356 of Cheap-GR-GFNs}
    \label{fig:dna_plots}
\end{figure}

\subsection{Fragment-Based Molecule Generation}\label{func:SA}
We continue with the chemistry-themed benchmarks, particularly with the Fragment-based molecule generation from \cite{frag, gflownet}, based on computational chemistry and Machine Learning, aimed at designing new molecules with desired properties. This task is significant in fields like pharmaceuticals, material science, and biochemistry, as new molecules can lead to new drugs, materials, and 
chemical processes. The objective functions are \texttt{SEH},\texttt{QED},\texttt{SA}, and \texttt{MW}, whose details can be found in Appendix \ref{appendix:frag}. We compare our methods Cheap-GR-GFNs and GR-GFNs with OP-GFNs, PC-GFNs and the Goal-Conditioned GFNs (GC-GFNs).
\subsubsection{Results}
We show the 3200 generated candidates in Figure \ref{fig:seh_plots}, and the evaluation metrics in Table \ref{tab:seh_table}. As with the metrics it is not clear to determine which method performs better; we also added a column indicating the following metric. We select all the Pareto fronts from each method as candidates, and we recompute the new Pareto front. Then, we count how many points for each method are not dominated. We show this in Figure \ref{fig:seh_pareto}. We see that for the cases of \texttt{QED-SA} and \texttt{SEH-SA}, our methods explore the \texttt{SA} function more than OP-GFNs, which tend to explore the other objective function. For \texttt{SEH-QED} and \texttt{SEH-MW}, our methods sample closer from the Pareto front. For the other two pairs, we do not find significant differences among the different methods because the Pareto front is easier to find.

\subsection{QM9}
A related challenge to Fragmentation-based molecule generation is the QM9 environment \cite{qm9}, where molecules are generated by sequentially adding atoms and bonds, with a maximum of 9 atoms. We explore objective functions such as \texttt{MXMNet} (HOMO-LUMO gap prediction), \texttt{logP}, \texttt{SA}, and a modified \texttt{MW} function, detailed in the Appendix \ref{appendix:QM9}

\subsubsection{Results}
The plots containing the 3200 generated candidates are presented in Figure \ref{fig:qm9_plots}, while the metrics used to evaluate them are shown in Table \ref{tab:qm9_table}. Again we compute the number of non-dominated samples as in the previous benchmark. Remarkably, these metrics show improvements in most of the objective functions. To support this point, we plot the different Pareto fronts in Figure \ref{fig:qm9_pareto_plots}.

\section{Conclusion}
We introduced a new approach to adapt GFlowNets for MOO tasks, developing PC-GFNs and OP-GFNs, which create weighted rewards and local orders among samples, respectively. We combined these methods to define a global reward that imposes a global order, hence defining Global-Order GFlowNets.

In our experiments, Cheap-GR-GFNs and GR-GFNs performed on par with or better than PC-GFNs and OP-GFNs, especially in the HyperGrid, DNA, and QM9 benchmarks, achieving results closer to the ideal Pareto front. Cheap-GR-GFNs outperformed PC-GFNs in several N-Grams tests, particularly in challenging 4-unigram and 4-bigram cases. In the Fragmentation-based molecule generation benchmark, our methods showed strong performance, with GR-GFNs and OP-GFNs sampling closer to the Pareto front in several instances.

While NN-GFNs excelled in the HyperGrid problem, capturing the full Pareto front in specific cases, their performance declined in other benchmarks. 
Although no algorithm proved better performance across all tests, our proposed methods represent a competitive alternative, excelling in specific scenarios.

The choice of the global order depends on the problem. In high-dimensional or real-world cases, we recommend using Cheap-GR-GFNs as they keep the complexity of OP-GFNs. This is especially useful in real-world applications where computational efficiency is important, but also the ability to navigate multiple objectives.
\subsection{Limitations and Future Work}
This work encompasses the concept of finding different global orders that are consistent with the local order given by the relation of Pareto dominance. Although we have discovered many different global orders, it is mandatory to remark that we have not found them all. More importantly, we have not established the optimality of any of our methods, and it remains possible that alternative global orders could outperform those presented in this work. Possible improvements in this matter are potential future approaches in this direction of research. We note that our methods do not surpass others in every aspect. However, in cases where they do not outperform, they either exhibit enhanced exploration capabilities (as shown in Figure \ref{fig:dna_plots}) or maintain competitive results. Finally, we encountered a significant challenge with Botorch \cite{botorch}, where a bug in the Pareto computation function directly impacted the retrieval of Pareto points, substantially hindering our progress. However, we are pleased to report that, after notifying the developers, the issue has been resolved. Experiments were conducted on a system with 8× RTX 6000 Ada GPUs, and 2× AMD EPYC 7302 CPUs.
\bibliographystyle{plainnat}
\bibliography{iclr2025_conference}


\appendix
\section{Global Orders}
We first prove that the Global Rank algorithm is consistent with the local relation of Pareto Dominance.
\begin{theorem}[Consistency of Global Rank with Pareto Dominance]\label{appendix:proof}
Let $\mathcal{X}$ be a set of samples with associated rewards $R_1, R_2, ..., R_d$, and let $\hat{R}: \mathcal{X} \to \mathbb{R}$ be a ranking function produced by the Global Rank algorithm. For any two samples $x_1, x_2 \in \mathcal{X}$, if $x_1$ Pareto dominates $x_2$, then $\hat{R}(x_1) > \hat{R}(x_2)$.
\end{theorem}
\begin{proof}
    The first affirmation to notice is that while $x_1$ is still in $D$ (i.e., it has not been assigned a rank yet), $x_2$ will be there as well, due to the fact that $x_2$ cannot be in the Pareto front as there is a point that dominates it ($x_1$), hence $\hat{R}(x_2)\leq \hat{R}(x_1)$. Also, it is important to notice that all points that dominate $x_1$, dominate $x_2$ as the Pareto dominance is a transitive relation. Therefore, when all the points that dominate $x_1$ have been assigned a rank, $x_1$ will be in the new Pareto front, so $\hat{R}(x_1) > \hat{R}(x_2)$. 
\end{proof}

\subsection{Comparison of Global Orders}
To illustrate how the ranks are distributed across samples, we present different examples in Figure \ref{fig:plot_comp}, where the data samples are the discretization of the square $ [0,1 ]\times[0,1]$ in squares of size $1/32$ and the rewards are $r: (x,y) \to (x, y)$, $r: (x,y) \to (x, \frac{y}{1 + x^2})$, $r: (x,y) \to (x, (e^{-(x-1/2)^2}+y)/2)$, and finally $r: (x,y) \to ((\pi x)cos(\pi x),(\pi y)sin(\pi y))$. We also show the distribution of the rewards across the different points of the square in Figure \ref{fig:rew_points}.

\begin{figure}[h]
    \centering
    \begin{subfigure}[b]{0.4\textwidth}
        \includegraphics[width=\textwidth]{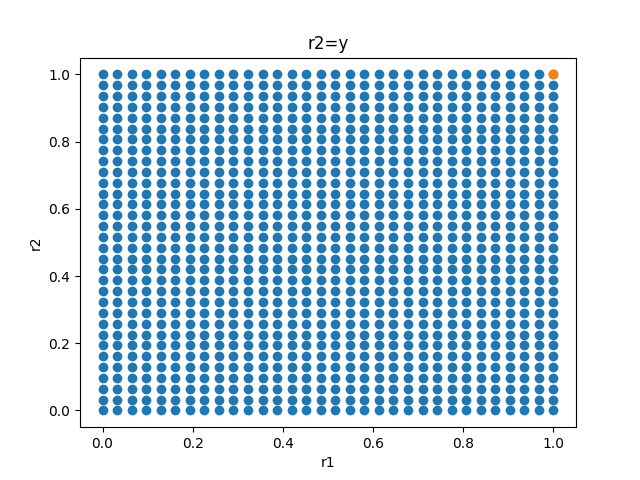}
        \label{fig:uniform}
    \end{subfigure}
     \begin{subfigure}[b]{0.4\textwidth}
        \includegraphics[width=\textwidth]{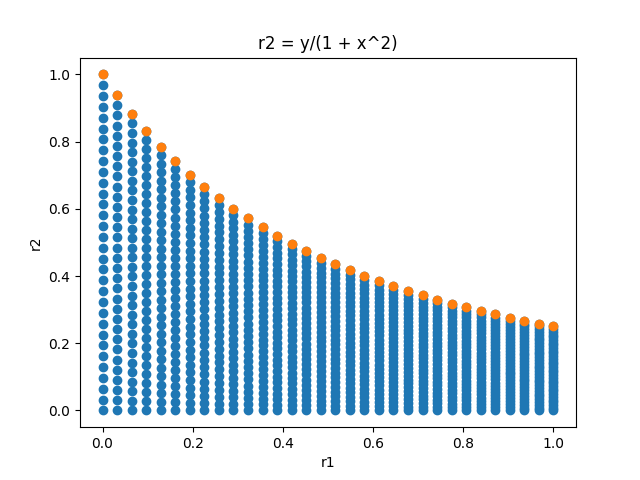}
        \label{fig:uniform}
    \end{subfigure}
    \hfill 
    \begin{subfigure}[b]{0.4\textwidth}
        \includegraphics[width=\textwidth]{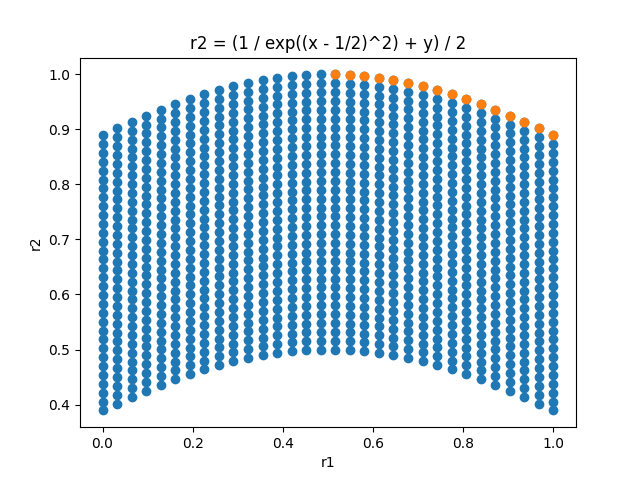}
        \label{fig:x_axis_exp}
    \end{subfigure}
    \begin{subfigure}[b]{0.4\textwidth}
        \includegraphics[width=\textwidth]{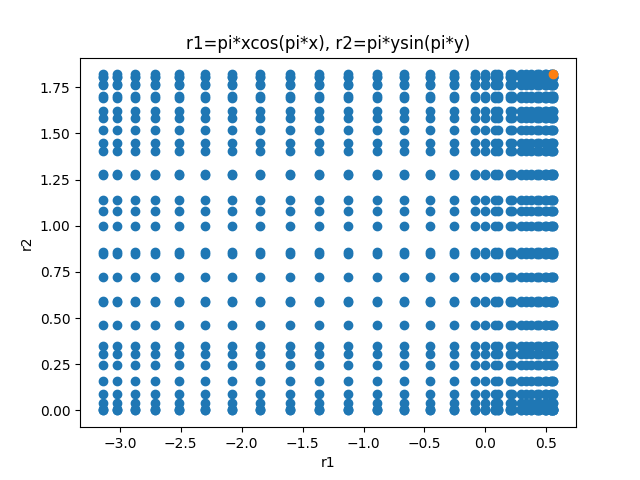}
        \label{fig:x_axis_exp}
    \end{subfigure}
    \caption{Image of the different points in the $[0,1]\times [0,1]$. In orange, we find the Pareto front. In the first cases, we omit $r_1(x,y)=x$ for clarity.}
    \label{fig:rew_points}
    \end{figure}

The first example to consider is the more straightforward case, where the rewards are the projections $r_1(x,y) = x$ and $r_2(x,y) = y$. In this setup, we observe that the image of the points is, trivially, uniformly distributed. Although the orders look similar, we can start to see some differences, especially when either $x$ or $y$ values are smaller. Similar behavior can be seen when $r: (x,y) \to (x, (e^{-(x-1/2)^2}+y)/2)$.

The case where we start to see some more notable differences is when using $r: (x,y) \to (x, \frac{y}{1 + x^2})$, as when $x$ increases, the images are remarkably more skewed towards the $x$ axis. As the Nearest Neighbor Order takes into account the Euclidean distance to the Pareto front whilst the Global Rank one does not, we observe differences above all when points are assigned a lower rank.

The most significant difference lies in the fourth example, where more complicated rewards ($r(x,y) = ((\pi x) cos(\pi x),(\pi y) sin(\pi y)))$ demonstrate that the assigned values in $\hat{R}$ can be completely different from each other.

\begin{figure}
    \centering
    \begin{subfigure}[b]{0.49\textwidth}
        \includegraphics[width=\textwidth]{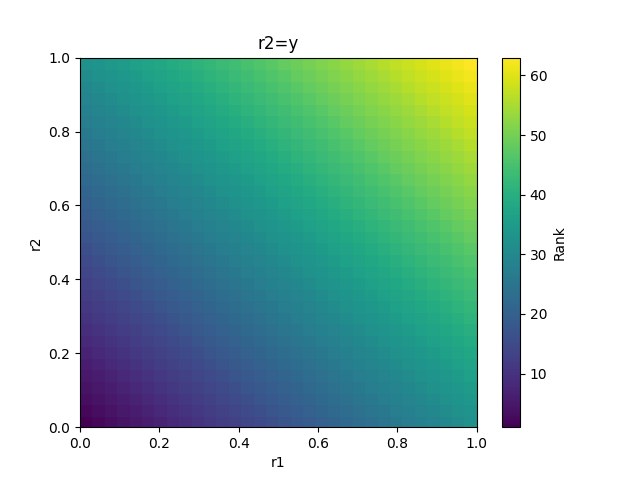}
        \label{fig:uniform}
    \end{subfigure}
     \begin{subfigure}[b]{0.49\textwidth}
        \includegraphics[width=\textwidth]{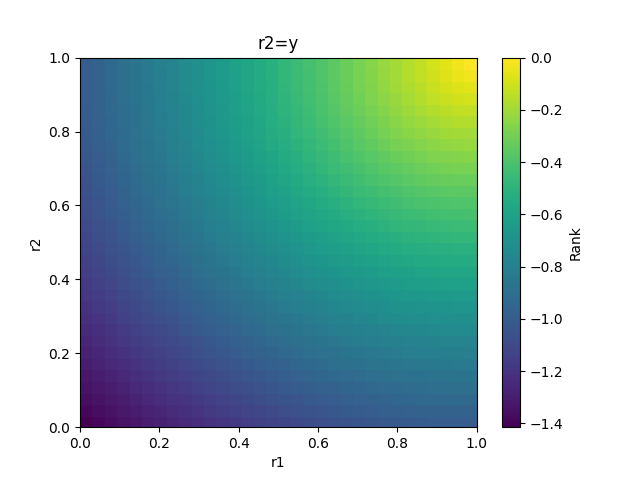}
        \label{fig:uniform}
    \end{subfigure}
    \hfill 
    \begin{subfigure}[b]{0.49\textwidth}
        \includegraphics[width=\textwidth]{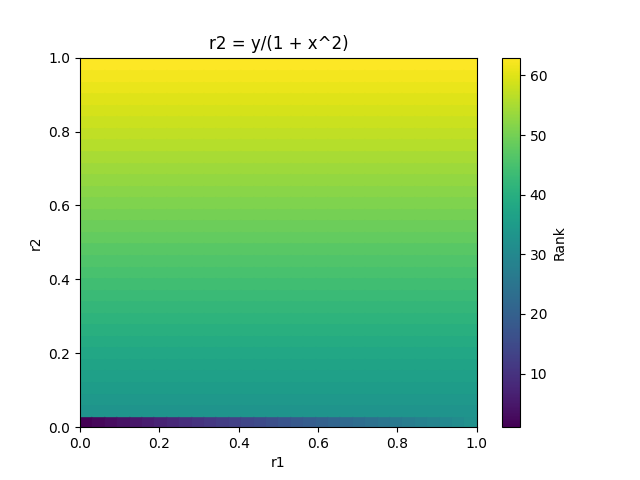}
        \label{fig:x_axis_exp}
    \end{subfigure}
    \begin{subfigure}[b]{0.49\textwidth}
        \includegraphics[width=\textwidth]{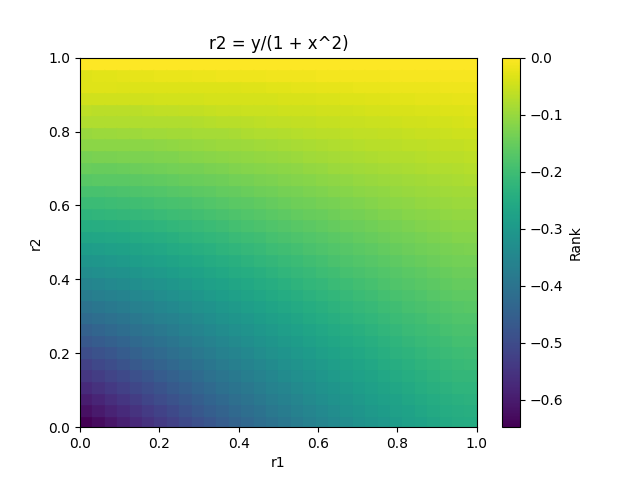}
        \label{fig:x_axis_exp}
    \end{subfigure}
    \hfill 

    \centering
    \begin{subfigure}[b]{0.49\textwidth}
        \includegraphics[width=\textwidth]{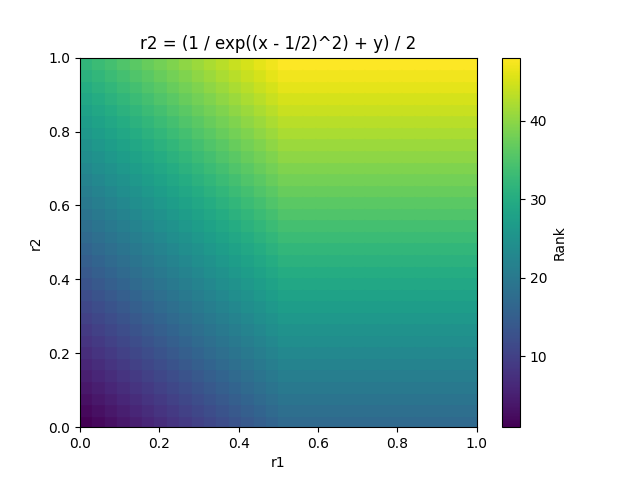}
        \label{fig:uniform}
    \end{subfigure}
    \begin{subfigure}[b]{0.49\textwidth}
        \includegraphics[width=\textwidth]{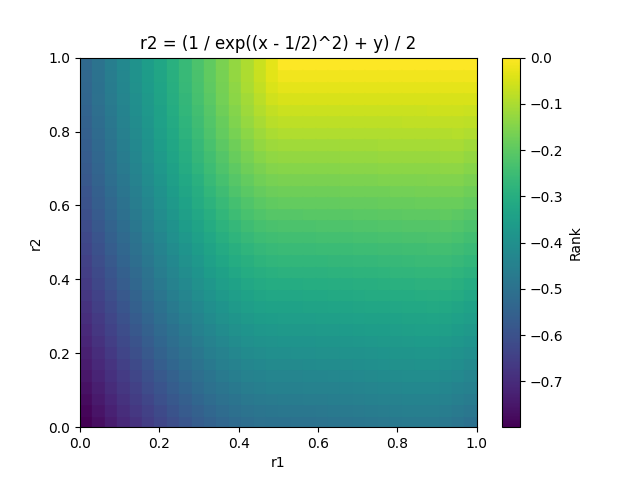}
        \label{fig:uniform}
    \end{subfigure}
    \hfill 
    \begin{subfigure}[b]{0.49\textwidth}
        \includegraphics[width=\textwidth]{examples_gr/cossin.png}
        \label{fig:x_axis_exp}
    \end{subfigure}
    \begin{subfigure}[b]{0.49\textwidth}
        \includegraphics[width=\textwidth]{examples_nn/cossin.png}
        \label{fig:x_axis_exp}
    \end{subfigure}
\caption{Comparison of the algorithms Global Rank (left) and Nearest Neighbor Order (right) with different rewards}
\label{fig:plot_comp}
\end{figure}
\newpage
\section{Evaluation Metrics}\label{sec:Metrics}

In evaluating the performance of our proposed algorithms, we utilize a range of metrics designed to capture various aspects of solution quality, diversity, and convergence to the true Pareto front. To define the metrics we are going to evaluate our algorithms with, we first denote the true Pareto front points $P$, the generated candidates by the network $S$, and finally $P$', the Pareto front of $S$. For some problems, the true Pareto front $P$ is not available or is unknown. However, since the functions are primarily normalized or bounded, we will use a discretization of the extreme faces of the hypercube as the reference set. The goal is to provide a comprehensive evaluation that accounts for both the closeness of generated solutions to the true Pareto front and the diversity of those solutions.

With the previous introduction, we introduce the following metrics:

\begin{itemize}
    \item \textbf{Inverted Generational Distance}:
    \[
    \text{IGD+}(S, P) := \frac{1}{|P|}\sum_{p\in P} (\text{min}_{s\in S} ||p - s|| ^2).
    \]
    This metric, introduced in \cite{GD+}, derives from previous metrics like the \textbf{Generational Distance} (GD) \cite{GD} and the \textbf{Inverted Generational Distance} (IGD) \cite{IGD}. GD calculates the average distance from each candidate $s\in S$ to the Pareto front $P$. IGD, in contrast, averages this quantity over the points of $P$ instead of the points of $S$. By focusing solely on the non-negative part of the distance, this method accounts for the directionality of improvements, avoiding penalties on solutions that exceed the Pareto front.

    \item \textbf{Averaged Hausdorff Distance (Plus)} is given by
    \[
    d_H^+(S,P) = \text{max}(\text{GD+}(S,P), \text{IGD+}(S,P)),
    \]
    where the \textbf{Generational Distance Plus} (GD+) is also introduced. Similar to IGD+, GD+ is defined as:
    \[
    \text{GD+} = \frac{1}{|S|}\sum_{s\in S} (\text{min}_{p\in P} ||p - s|| ^2).
    \]
    The standard Averaged Hausdorff Distance is 
    \[
    d_H(S,P)=\text{max}(\text{GD}(S,P), \text{IGD}(S,P)).
    \]

    \item \textbf{Hypervolume Indicator} (HV) \cite{HV} measures the quality of a set of solutions by computing the volume covered by the union of rectangles formed between the solutions provided by the network and a reference point. $\text{HV}(S,r)$ can be expressed as the Lebesgue measure of the union of rectangles formed by each point $s \in S$ and a reference point $r$, such that each rectangle is defined by the product of the intervals $[s_i, r_i]$ for $i \in \{1,...,d\}$.

    \item \textbf{Pareto-Clusters Entropy} (PC-ent) \cite{goal_cond_gflownets} metric focuses on the diversity of the generated (Pareto front) samples. Inspired by the formulation of entropy $H(X) = -\sum_{x\in X}x\log(x)$, this metric first creates clusters $P'_j$ with points in $P'$ based on distances to the reference Pareto front $P$. The PC-ent metric is defined as:
    \[
    \text{PC-ent}(P,P') = -\sum_{j}\frac{|P'_j|}{|P|}\log\frac{|P'_j|}{|P|}.
    \]

    \item \textbf{$\text{R}_2$ Indicator} \cite{r2} utilizes a set of uniformly distributed reference vectors alongside a utopian point \(z^*\). We introduce the Uniform Reference Vectors $\Lambda$, which collects uniformly distributed vectors across the objective space, capturing all its directions. The formula for the $\text{R}_2$ Indicator is:
    \[
    \text{R}_2(S,\Lambda, z^*) = \frac{1}{|\Lambda|} \sum_{\lambda \in \Lambda} \min_{s \in S} \max_{i \in \{1, \ldots, d\}} \{ \lambda_i \cdot |z^*_i - s_i| \}.
    \]

    \item \textbf{Pareto Coverage} is used when the true Pareto front $P$ is available, and it is defined as the proportion of the Pareto front that has been seen in $S$ (or equivalently in $P'$).

    \item \textbf{Samples in Front} is defined as the proportion of generated samples that are in the true Pareto front when the true Pareto front is available.

    \item \textbf{Top-k Divergence} quantifies the diversity among the best samples by measuring how different these samples are from each other, providing insights into the exploratory success of the sampling process in generating diverse high-quality candidates.
\end{itemize}
\section{Experiments}\label{appendix:experiments}
Across all this section we refer to our experiments in the tables with a "*". Our methods are Global-Rank GFlowNets (GR-GFNs), Trimmed Global-Rank GFlowNets when we assign a maximum rank $k$ (GR-GFNs ($k$)), Cheap Global-Rank GFlowNets (Cheap GR-GFNs), Nearest neighbor GFlowNets (NN-GFNs) and their version where we linearly interpolate the Pareto Front (NN-int-GFNs).
\subsection{HyperGrid}
\label{appendix:hypergrid}

We investigate the synthetic HyperGrid environment introduced in \cite{gflownet}. In this setup, states $S$ form a $d$-dimensional grid with side length $H$, where each state is defined as $\bigl\{(s_1, \dots, s_d) \mid s_i \in \{1, \dots, H\}, i \in \{1, \dots, d\}\bigr\}$. The environment allows actions that increment one coordinate without leaving the grid or stopping at a state. We evaluate four different reward functions: \texttt{branin}, \texttt{currin}, \texttt{shubert}, and \texttt{beale}, chosen for their diverse sparsity patterns in the Pareto front. Experiments are conducted with $d \in \{2, 3\}$ and $H = 32$.

The objective functions are defined as follows:
\[
\texttt{branin}(x_1, x_2) = a\left(x_2 - bx_1^2 + cx_1 - r\right)^2 + s(1-t)\cos(x_1) + s,
\]
where $a = 1, \quad b = \frac{5.1}{4\pi^2}, \quad c = \frac{5}{\pi}, \quad r = 6, \quad s = 10, \quad t = \frac{1}{8\pi}$ and $x_1, x_2$ are scaled as $x_1 \gets 15x_1 - 5$ and $x_2 \gets 15x_2$

\[
\texttt{currin}(x_1, x_2) = \left(1 - e^{-0.5x_2}\right) \cdot \frac{2300x_1^3 + 1900x_1^2 + 2092x_1 + 60}{13.77(100x_1^3 + 500x_1^2 + 4x_1 + 20)}
\]

\[
\texttt{shubert}(x_1, x_2) = \frac{\sum_{i=1}^{5} i \cos((i+1)x_1 + i) \sum_{i=1}^{5} i \cos((i+1)x_2 + i)}{397} + \frac{186.8}{397}
\]

\[
\texttt{beale}(x_1, x_2) = \frac{(1.5 - x_1 + x_1x_2)^2 + (2.25 - x_1 + x_1x_2^2)^2 + (2.625 - x_1 + x_1x_3^2)^2}{38.8}
\]

For $d = 2$, we utilize Global Rank GFlowNets (GR-GFNs) and Nearest Neighbor GFlowNets (NN-GFNs), including both basic and interpolated versions. For $d = 3$, we also employ Global Rank (Trimmed) GFlowNets with a maximum rank of 25. Comparisons are made against Preference-Conditional GFlowNets (PC-GFNs) and Order-Preserving GFlowNets (OP-GFNs).

The learned GFlowNet sampler is used to generate 1280 candidates for evaluation.

\subsubsection{Network Structure and Training Details}
As in the rest of experiments, we use the guidance of \cite{OP-Gfns}. The backward transition probability $P_B$ is set to be uniform across all states. The forward transition probability $P_F$ is parameterized by a 3-layer MLP, each with 64 hidden units and LeakyReLU as the activation function. The Adam optimizer is employed for training, using a learning rate of 0.01 and a batch size of 128. The model is trained over 1000 steps. Following guidance from \cite{moo_gfn}, for PC-GFNs the weight vector $w$ is drawn from a $\text{Dirichlet}(1.5)$ distribution, and $\beta$ is sampled from a $\Gamma(16, 1)$ distribution during training. For evaluation, $w$ is sampled from the same $\text{Dirichlet}(1.5)$ distribution, but $\beta$ is fixed at 16.

\subsubsection{Results}
We first present the results for $d=2$. Due to the low dimension, we can comfortably plot the results. First, we display the generated Pareto fronts in Figure \ref{fig:grid-paretofronts2d}, and then, we show how the different generated sample rewards are distributed in Figure \ref{fig:grid_values}. The results for the different metrics are summarized in Tables \ref{tab:grid_benchmark_results_2d}, \ref{tab:grid_benchmark_results}.

From Figure \ref{fig:grid_values} we can observe that GR-GFNs accomplish having most of the generated samples in the Pareto front, in contrast with NN-Gto NN-GFNs that, althoughg a lot of samples in the Pareto front, also havealso have manybserve that, especially for the case of the \texttt{shubert} functions, GR-GFNs generate much better results than the previous GFNs. We observe that in terms of $d_H$ and Samples in front, GR-GFNs greatly surpass the other methods, whilst in IGD+ and Pareto Coverage, NN-GFNs are the best ones.

Our methods, compared to OP-GFNs, are very similar in terms of Pareto coverage, except for the cases with \texttt{branin-currin} and \texttt{branin-shubert}, where the isolated point in the top-right corner is only achieved with the NN-GFNs (both basic and interpolated).

We now present the results for $d=3$ in Figure \ref{fig:grid_all_points_3d}. Even with the expansion of a dimension, we still see that our methods outperform OP-GFNs and PC-GFNs in most of the metrics, especially in Samples in front where we appreciate the greatest difference with respect to the rest in the case of GR-GFNs. In fact, only in the Pareto coverage of the \texttt{branin-shubert-beale} functions, OP-GFNs are slightly better than the rest.

\begin{figure}
    \centering
    \includegraphics[width=0.9\textwidth]{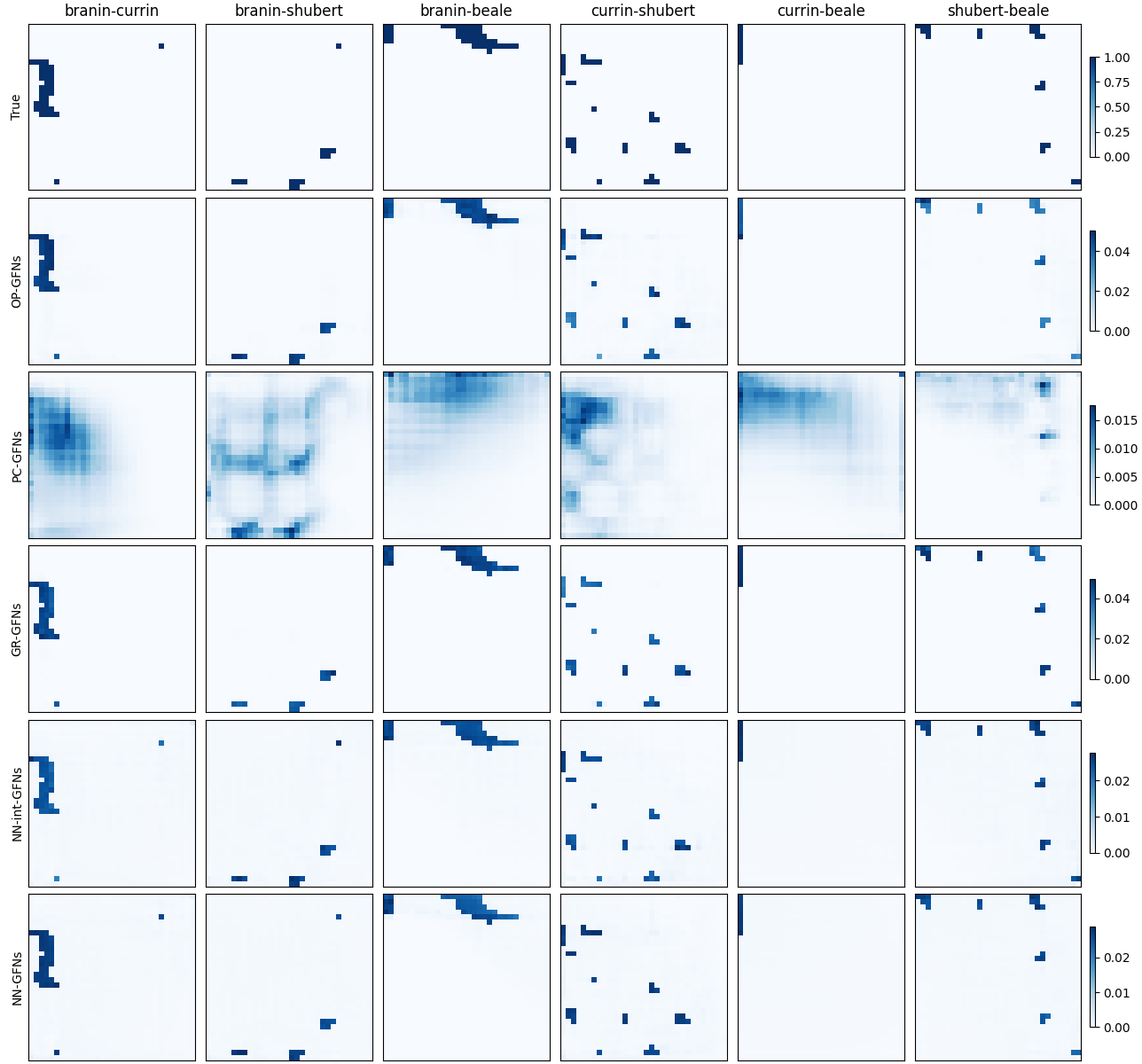}
    \caption{In the top row, we plot the indicator function of the true Pareto front. In the other rows, we plot the learned reward distribution of the different methods.}
    \label{fig:grid-paretofronts2d}
\end{figure}
\begin{figure}
    \centering
    \includegraphics[width=1\textwidth]{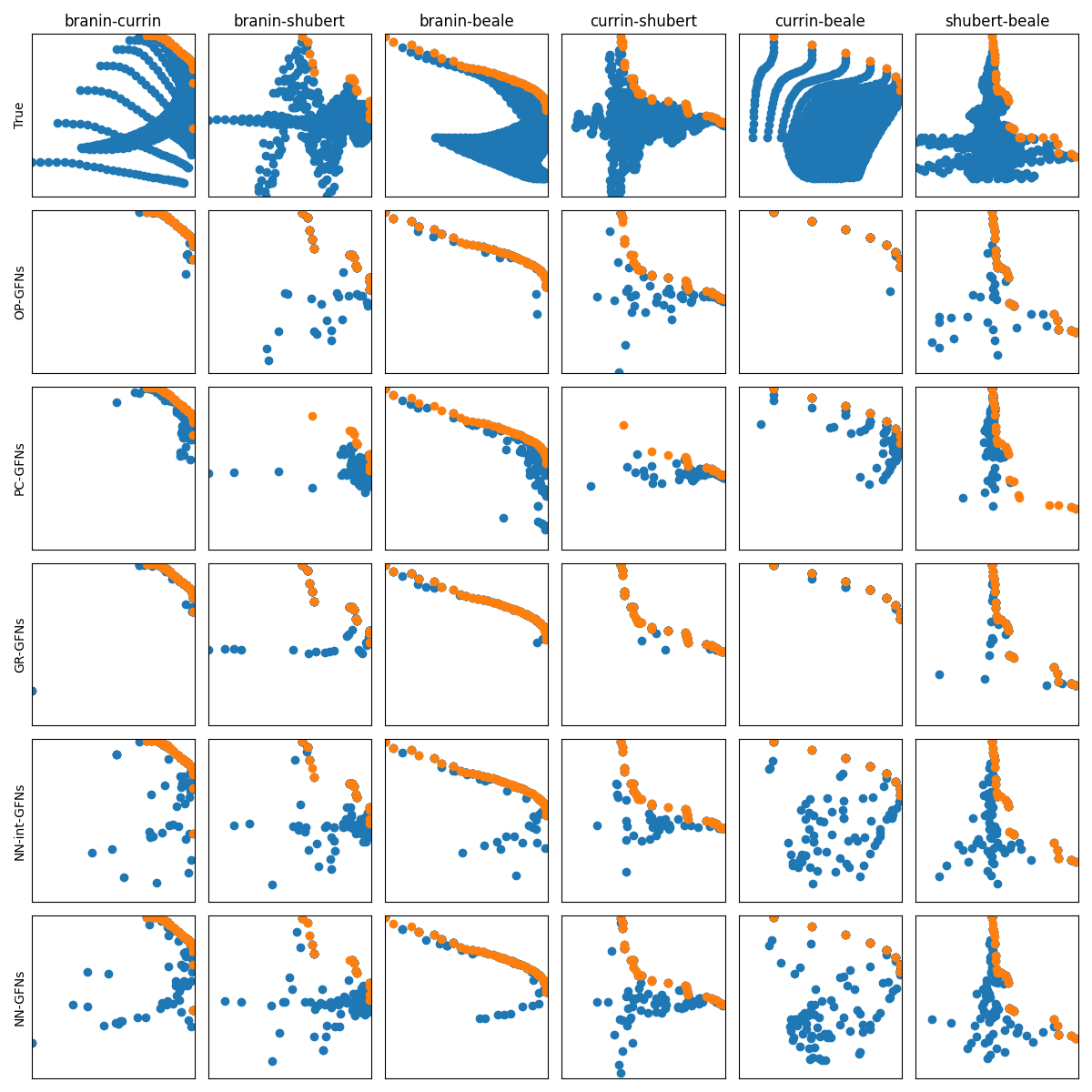}
    \caption{128 generated candidates (blue) and their respective Pareto front (orange). The first row, being the ground truth, is the image of all possible values of the discretized grid}
    \label{fig:grid_values}
\end{figure}
\begin{table}
\caption{HyperGrid Results for $d=2$}
\centering
\sisetup{
  table-format=2.2,
  detect-weight=true,
  detect-inline-weight=math
}
\begin{tabular}{
  l
  S[table-format=1.3]
  S[table-format=1.2e+2]
  S[table-format=1.3]
  S[table-format=1.3]
  S[table-format=1.3]
}
\toprule
{\textbf{Method}} & {d$_H(P',P)$ (\down)} & {IGD+ (\down)} & {PC-ent (\up)} & {\% front (\up)} & {Coverage (\up)} \\
\midrule
\textbf{branin-currin} \\
\midrule
OP-GFNs           & 0.014 & 2.66e-06 & 3.555 & 0.889 & 0.972 \\
PC-GFNs           & 0.041 & 2.66e-06 & 3.555 & 0.251 & 0.972 \\
GR-GFNs*           & \textbf{0.008} & 2.66e-06 & 3.555 & \textbf{0.935} & 0.972 \\
NN-int-GFNs*       & 0.06  & \textbf{1.18e-08} & \textbf{3.584} & 0.675 & \textbf{1} \\
NN-GFNs*           & 0.063 & \textbf{1.18e-08} & \textbf{3.584} & 0.669 & \textbf{1} \\
\midrule
\textbf{branin-shubert} \\
\midrule
OP-GFNs           & 0.054 & 6.82e-06 & 2.565 & 0.736 & 0.929 \\
PC-GFNs           & 0.078 & 2.60e-02 & 2.303 & 0.052 & 0.714 \\
GR-GFNs*          & \textbf{0.028} & 6.82e-06 & 2.565 & \textbf{0.823} & 0.929 \\
NN-int-GFNs*    & 0.098 & \textbf{9.00e-09} & \textbf{2.639} & 0.402 & \textbf{1} \\
NN-GFNs*           & 0.095 & \textbf{9.00e-09} & \textbf{2.639} & 0.412 & \textbf{1} \\
\midrule
\textbf{branin-beale} \\
\midrule
OP-GFNs           & 0.005 & \textbf{1.38e-08} & \textbf{3.784} & 0.899 & \textbf{1} \\
PC-GFNs           & 0.065 & \textbf{1.38e-08} & \textbf{3.784} & 0.225 & \textbf{1} \\
GR-GFNs*           & \textbf{0.002} & \textbf{1.38e-08} & \textbf{3.784} & \textbf{0.95}  & \textbf{1} \\
NN-int-GFNs*       & 0.031 & \textbf{1.38e-08} & \textbf{3.784} & 0.784 & \textbf{1} \\
NN-GFNs*           & 0.018 & \textbf{1.38e-08} & \textbf{3.784}& 0.832 & \textbf{1} \\

\midrule
\textbf{currin-shubert} \\
\midrule
OP-GFNs           & 0.036 & \textbf{1.29e-08} & \textbf{3.466} & 0.727 & \textbf{0.941} \\
PC-GFNs           & 0.038 & 3.60e-02 & 3.000 & 0.127 & 0.529 \\
GR-GFNs*           & \textbf{0.010} & \textbf{1.29e-08} & \textbf{3.466} & \textbf{0.889} & \textbf{0.941} \\
NN-int-GFNs*       & 0.049 & \textbf{1.29e-08} & \textbf{3.466} & 0.63  & \textbf{0.941} \\
NN-GFNs*           & 0.055 & \textbf{1.29e-08} & \textbf{3.466} & 0.601 & \textbf{0.941} \\
\midrule
\textbf{currin-beale} \\
\midrule
OP-GFNs           & 0.002 & \textbf{1.57e-08} & \textbf{2.079} & \textbf{0.977} & \textbf{1} \\
PC-GFNs           & 0.078 &\textbf{1.57e-08} & \textbf{2.079} & 0.203 & \textbf{1} \\
GR-GFNs*           & \textbf{0.001} & \textbf{1.57e-08} & \textbf{2.079} & 0.974 & \textbf{1} \\
NN-int-GFNs*       & 0.023 & \textbf{1.57e-08} & \textbf{2.079} & 0.434 & \textbf{1} \\
NN-GFNs*           & 0.259 & \textbf{1.57e-08} & \textbf{2.079} & 0.379 & \textbf{1} \\

\midrule
\textbf{shubert-beale} \\
\midrule
OP-GFNs           & 0.038 & \textbf{7.10e-09} & \textbf{3.091} & 0.757 & \textbf{0.733} \\
PC-GFNs           & 0.033 & 1.70e-02 & 3.015 & 0.362 & 0.667 \\
GR-GFNs*           & \textbf{0.008} & \textbf{7.10e-09} & \textbf{3.091} & \textbf{0.889} & \textbf{0.733} \\
NN-int-GFNs*       & 0.067 & \textbf{7.10e-09} & \textbf{3.091} & 0.523 & \textbf{0.733} \\
NN-GFNs*          & 0.068 & \textbf{7.10e-09} & \textbf{3.091} & 0.521 & \textbf{0.733} \\
\bottomrule
\end{tabular}

\label{tab:grid_benchmark_results_2d}
\end{table}
\begin{table}
\caption{HyperGrid results for $d=3$}
\centering
\sisetup{
  table-format=2.2,
  detect-weight=true,
  detect-inline-weight=math
}
\begin{tabular}{
  l
  S[table-format=1.3]
  S[table-format=1.2e+2]
  S[table-format=1.3]
  S[table-format=1.3]
  S[table-format=1.3]
}
\toprule
{\textbf{Method}} & {d$_H(P',P)$ (\down)} & {IGD+ (\down)} & {PC-ent (\up)} & {\% front (\up)} & {Coverage (\up)} \\
\midrule
\textbf{branin-currin-shubert} \\
\midrule
OP-GFNs           & 0.025 & 9.07e-07 & 4.630  & 0.786 & \textbf{0.991} \\
PC-GFNs           & 0.055 & 0.0023   & 4.510  & 0.331 & 0.860  \\
GR-GFNs*           & 0.014 & 9.07e-07 & 4.630  & 0.907 & 0.963 \\
GR-GFNs (25)*      & \textbf{0.013} & 9.07e-07 & 4.630  & \textbf{0.929} & \textbf{0.991} \\
NN-int-GFNs*       & 0.056 & \textbf{1.61e-08} & \textbf{4.640}  & 0.64  & \textbf{0.991} \\
NN-GFNs*           & 0.025 & 9.071e-07 & 4.630  & 0.827 & 0.963 \\
\midrule
\textbf{branin-currin-beale} \\
\midrule
OP-GFNs           & 0.017 & 2.00e-04 & 5.411 & 0.746 & 0.978 \\
PC-GFNs           & 0.017 & 8.00e-04 & 5.305 & 0.723 & 0.882 \\
GR-GFNs*           & \textbf{0.003} & \textbf{1.64e-05} & \textbf{5.425} & \textbf{0.969} & \textbf{0.991} \\
GR-GFNs (25)*      & \textbf{0.003} & 6.46e-05 & \textbf{5.425} & 0.963 & \textbf{0.991} \\
NN-int-GFNs*       & 0.011 & 1.00e-04 & \textbf{5.425} & 0.860 & \textbf{0.991} \\
NN-GFNs*           & 0.011 & 1.00e-04 & 5.421 & 0.934 & 0.987 \\

\midrule
\textbf{branin-shubert-beale} \\
\midrule
OP-GFNs           & 0.024 & 2.00e-04 & 4.984 & 0.685 & \textbf{0.987} \\
PC-GFNs           & 0.041 & 2.80e-02 & 4.650 & 0.355 & 0.686 \\
GR-GFNs*           & 0.010 & \textbf{1.45e-08} & \textbf{4.990} & 0.896 & 0.967 \\
GR-GFNs (25)*      & \textbf{0.004} & \textbf{1.45e-08} & \textbf{4.990} & \textbf{0.928} & 0.967 \\
NN-int-GFNs*       & 0.028 & \textbf{1.45e-08} & \textbf{4.990} & 0.704 & 0.967 \\
NN-GFNs*           & 0.024 & 1.00e-04 & 4.984 & 0.817 & 0.961 \\

\midrule
\textbf{currin-shubert-beale} \\
\midrule
OP-GFNs           & 0.023 &  \textbf{1.41e-08} & \textbf{4.844} & 0.690 & \textbf{1} \\
PC-GFNs           & 0.037 & 4.00e-02 & 4.025 & 0.340 & 0.411 \\
GR-GFNs*           & \textbf{0.009} &  \textbf{1.41e-08} & \textbf{4.844} & 0.902 & 0.992 \\
GR-GFNs (25)*      & \textbf{0.009} &  \textbf{1.41e-08} & \textbf{4.844} & \textbf{0.908} & 0.992 \\
NN-int-GFNs*       & 0.035 &  \textbf{1.41e-08} & \textbf{4.844} & 0.673 & \textbf{1} \\
NN-GFNs*           & 0.020 &  \textbf{1.41e-08} & \textbf{4.844} & 0.830 & 0.992 \\
\bottomrule
\end{tabular}

\label{tab:grid_benchmark_results}
\end{table}

\begin{figure}
    \centering
    \includegraphics[width=1\textwidth]{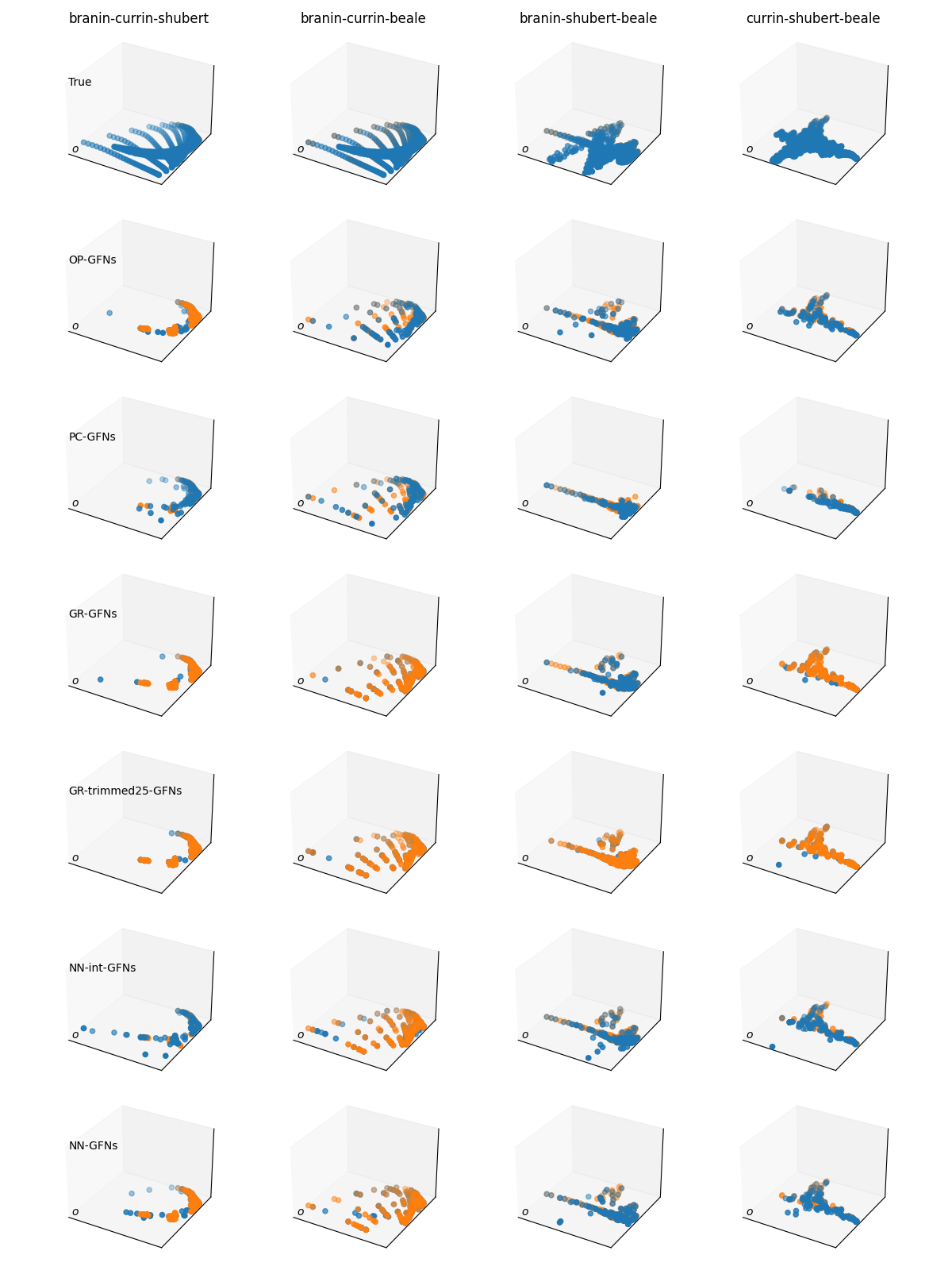}
    \caption{Results for $d=3$}
    \label{fig:grid_all_points_3d}
\end{figure}
\clearpage
\subsection{N-Grams}\label{appendix:ngrams}
This next synthetic benchmark is called N-Grams, and it was proposed by \cite{ngrams}. The goal here is to generate strings of a maximum length $L$ that are rewarded better if they have more occurrences of some given substrings, in this case individual letters (unigrams) or pairs (bigrams). We consider the objectives shown in Table \ref{tab:ngrams_objectives}. For the unigrams case, we set $L=18$ and for the bigrams, $L=36$.
\begin{table}[h]
\caption{Objectives considered for the n-grams task}

\centering
\begin{tabular}{|c l l|}
\hline
\# Objectives & Unigrams & Bigrams \\
\hline
2 & A, C & AC, CV \\
3 & A, C, V & AC, CV, VA \\
4 & A, C, V, W & AC, CV, VA, AW \\
\hline
\end{tabular}
\label{tab:ngrams_objectives}
\end{table}
In this benchmark, we saw that NN-GFNs were already underperforming compared to the other three main algorithms, and therefore, we leave it out in this study. Hence we compare OP-GFNs, PC-GFNs and the Cheap-GR-GFNs, with replay buffer of capacity 10000 and warm-up 1000.
\subsubsection{Network Structure and Training Details}
As in the previous benchmark, the backward transition probability \(P_B\) is configured to be uniform, whereas the forward transition probability $P_F$ is modeled using a Transformer-based encoder. This encoder is implemented featuring three hidden layers each with a dimension of 64 and utilizing eight attention heads for the embedding of the current state \(s\). It is characterized by being unidirectional with no dropout. For the PC-GFNs, as indicated in \cite{moo_gfn,OP-Gfns}, the preference is encoded using Dir(1), 50 bins, and the exponent $\beta$ for the reward being 96. We use Adam optimizer. In this case, with learning rates $10^{-4}$ and $10^{-3}$ for $P_F$ and $Z$ respectively.
\subsubsection{Results}
After training the network, we generate 1280 candidates and the results are summarized in Table \ref{tab:ngrams-experiment-results}, where the $k$ in the Top $k$-diversity is 10. We observe that in this particular case the two algorithms that perform better are PC-GFNs and Cheap-GR-GFNs. Except for the case of 2 unigrams (PC-GFNs and Cheap-GR-GFNs perform similarly) and the 2 and 3 bigrams (PC-GFNs perform better), we observe our method to outperform the others. We remark that our methods, even having similar results for two and three objectives, stand out with 4 objectives.
\begin{table}[h]
\caption{N-Grams Results}
\centering
\sisetup{
  table-format=2.2,
  detect-weight=true,
  detect-inline-weight=math
}
\begin{tabular}{
  l
  S[table-format=1.2]
  S[table-format=2.2]
  S[table-format=1.2]
  S[table-format=1.2]
  S[table-format=2.2]
}
\toprule
\textbf{REGEX} & {HV (\up)} & {R$_2$ (\down)} & {PC-ent (\up)} & {d$_H(P',P)$ (\down)} & {Diversity (\up)} \\
\midrule

\multicolumn{6}{l}{\textbf{2 Unigrams}} \\
\midrule
OP-GFNs           &\textbf{0.47}&1.46&2.25&0.34&7.17\\
PC-GFNs           &\textbf{0.47} &\textbf{\;\;1.45}  & \textbf{2.26} & \textbf{0.33} & 3.40  \\
Cheap-GR-GFNs* & \textbf{0.47} &\textbf{\;\;1.45}  & \textbf{2.26} & \textbf{0.33} & \textbf{\;\;8.22} \\
\midrule
\multicolumn{6}{l}{\textbf{2 Bigrams}} \\
\midrule
OP-GFNs           & 0.53 & 1.66  & 1.90 & 0.34 & 5.14 \\
PC-GFNs           & 0.52 & \textbf{\;\;1.42}  & \textbf{2.15} & \textbf{0.30} & \textbf{14.99} \\
Cheap-GR-GFNs* & \textbf{0.57} & 1.50  & 2.08 & 0.31 & 5.76 \\
\midrule
\multicolumn{6}{l}{\textbf{3 Unigrams}} \\
\midrule
OP-GFNs           & 0.13 & 11.07 & 3.69 & 0.59 & 9.74 \\
PC-GFNs           & 0.04 & 10.30  & 2.99 & 0.60 & 5.22 \\
Cheap-GR-GFNs* & \textbf{0.14} & \textbf{9.53}  & \textbf{3.88} & \textbf{0.57} & \textbf{\;\;9.87} \\
\midrule
\multicolumn{6}{l}{\textbf{3 Bigrams}} \\
\midrule
OP-GFNs           & 0.30  & 10.32 & 1.10  & 0.56 & 1.08 \\
PC-GFNs           & 0.32 & \textbf{\;\;8.38}  & \textbf{2.25} & \textbf{0.42} & \textbf{13.98} \\
Cheap-GR-GFNs* & \textbf{0.33} & 9.21  & 2.20  & 0.49 & 10.15 \\
\midrule
\multicolumn{6}{l}{\textbf{4 Unigrams}} \\
\midrule
OP-GFNs           & \textbf{0.03} & 53.79 & 4.78 & 0.79 & 11.30 \\
PC-GFNs           & 0.01 & 49.77 & 3.36 & 0.79 & 4.49 \\
Cheap-GR-GFNs* & \textbf{0.03} & \textbf{45.54} & \textbf{5.07} & \textbf{0.76} & \textbf{11.44} \\
\midrule
\multicolumn{6}{l}{\textbf{4 Bigrams}} \\
\midrule
OP-GFNs           & 0.06 & 50.18 & 3.89 & 0.68 & \textbf{16.88} \\
PC-GFNs           & 0.05 & 48.09 & 3.90  & 0.67 & 15.13 \\
Cheap-GR-GFNs* & \textbf{0.09} & \textbf{39.94} & \textbf{4.52} & \textbf{0.60}  & 12.79 \\
\bottomrule
\end{tabular}

\label{tab:ngrams-experiment-results}
\end{table}
\newpage
\subsection{DNA Sequence Generation}\label{appendix:dna}
\subsubsection{Objective Functions}
With a fixed length of 30 elements, we can compute several rewards:
\begin{itemize}
    \item \texttt{energy}: Free energy of the secondary structure computed with NUPACK \cite{Nupack}
    \item \texttt{pairs}: Number of base pairs
    \item \texttt{pins}: DNA hairpin index
\end{itemize}
Due to a time limitation we could only evaluate \texttt{energy-pins} and \texttt{energy-pins-pairs}, and it is left for the future to evaluate the other combinations of objective functions.
\subsubsection{Network Structure and Training Details}
We are going to compare ourselves with the same algorithms as before and also the same parameters as the N-Grams task, due to the similar nature of the problem. Following the advice of \cite{moo_gfn, OP-Gfns}, we now set $\beta = 80$ for PC-GFNs. Network architectures and training setup are identical to the previous experiment.

\subsection{Fragment-Based Molecule Generation}\label{appendix:frag}
\subsubsection{Objective Functions}
We use the same rewards adopted in the previous works \cite{moo_gfn, OP-Gfns, goal_cond_gflownets}:
\begin{itemize}
    \item \texttt{SEH}: A pretrained model acts as a proxy to predict the binding energy of a molecule to the soluble epoxide hydrolase,  closely related to the Alzheimer's treatment \cite{alzheimer}.
    \item \texttt{QED}: A measure for a drug's likeness \cite{drug_likeness}.
    \item \texttt{SA}: Synthetic Accessibility \cite{SA}. SA is extracted from RDKit \cite{RDKit}, and the final reward is $\text{R}_{SA} = (10 - \texttt{SA})/9$.
    \item \texttt{MW}: Molecular Weight, in this case a region that favors weights under 300: $\text{R}_{MW} = ((300 - \texttt{MW}) / 700 + 1).\text{clip}(0, 1)$
\end{itemize}
\subsubsection{Network Structure and Training Details}
As usual, $P_B$ is not parameterized but set as uniform. The novelty lies in $P_F$, which is now modeled by a Graph Neural Network (GNN) based on the graph transformer architecture \cite{gfn}. It has two layers with node embedding size of 64. As indicated in \cite{OP-Gfns}, the preference vector $w$ is represented using thermometer encoding with 16 bins, while the temperature $ \beta$ is similarly encoded but with 32 bins. The training details are very similar to previous sections, except for the fact that GR-GFNs are trained in 10000 steps instead of 20000. Due to limited computational resources, we set 30 the maximum rank in GR-GFNs.
\subsubsection{Ablation study}
In training with a replay buffer, there is a question of how to balance samples on the Pareto front and general samples. To address this, we conduct an experiment utilizing the \texttt{SEH-QED} objective functions, examining sample ratios of 0.1, 0.2, and 0.4. The outcomes of this investigation are shown in Figure \ref{fig:pareto_sizes}.
Our analysis indicates that setting the Pareto ratio to 0.1 results in generated samples that are highly aligned with the desired optimization objectives, closely resembling the ideal Pareto front. Conversely, setting the ratio to 0.4 yields less favorable outcomes. Specifically, this higher ratio appears to constrict the network's exploratory capabilities, leading to premature convergence on previously recognized Pareto front points during the initial stages of training (which, of course, are worse than the ideal Pareto front). This observation emphasizes the importance of carefully selecting the sample ratio to balance effective exploration with optimal convergence.
\begin{figure}[h!]
    \centering
    \includegraphics[width=0.4\textwidth]{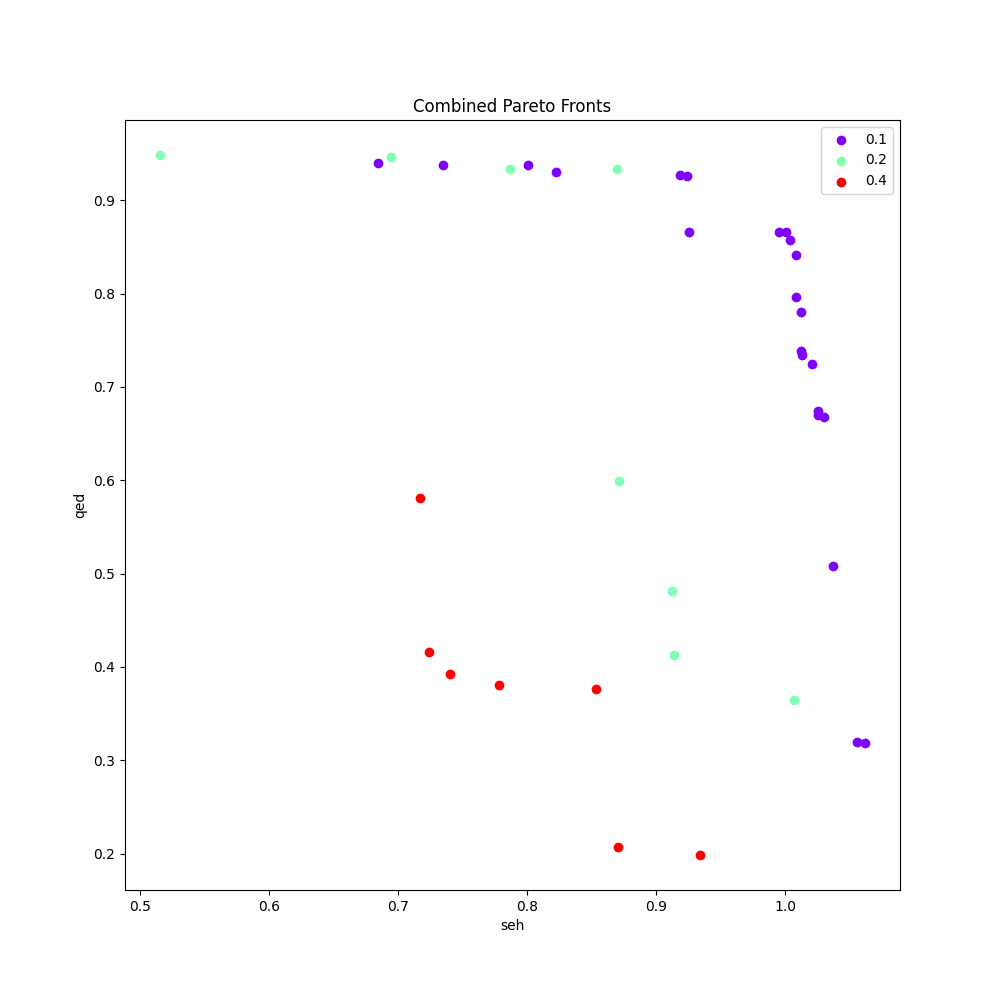}
    \caption{Ablation study of the Pareto ratio sizes}
    \label{fig:pareto_sizes}
\end{figure}
\subsubsection{Plots and Tables}
\begin{figure}[h]
    \centering
    \includegraphics[width=1\textwidth]{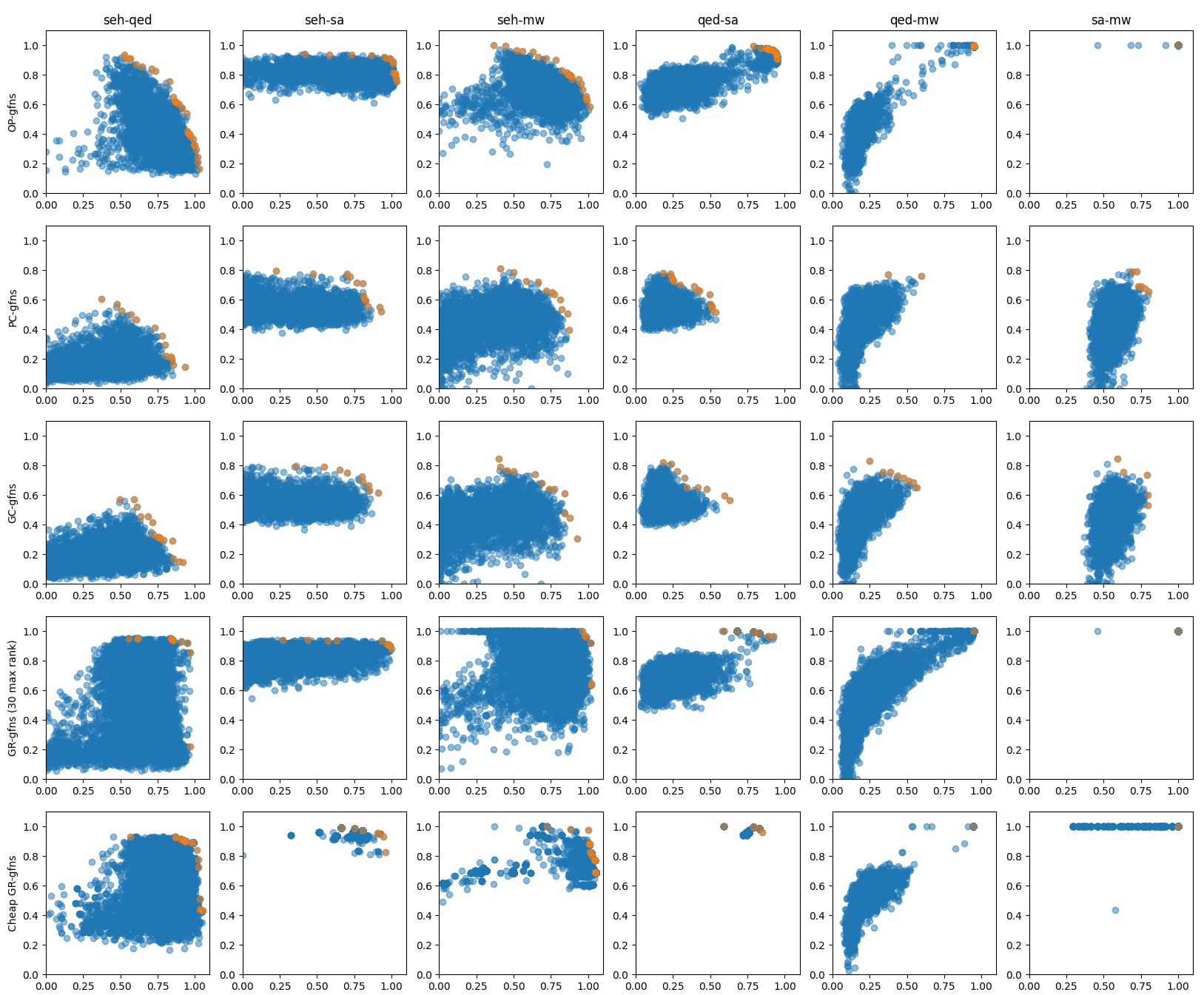}
    \caption{Results for the fragmentation-based molecule generation: generated samples (blue) and the Pareto fronts (orange)}
    \label{fig:seh_plots}
\end{figure}
\begin{figure}
    \centering
    \begin{subfigure}[b]{0.49\textwidth}
        \includegraphics[width=\textwidth]{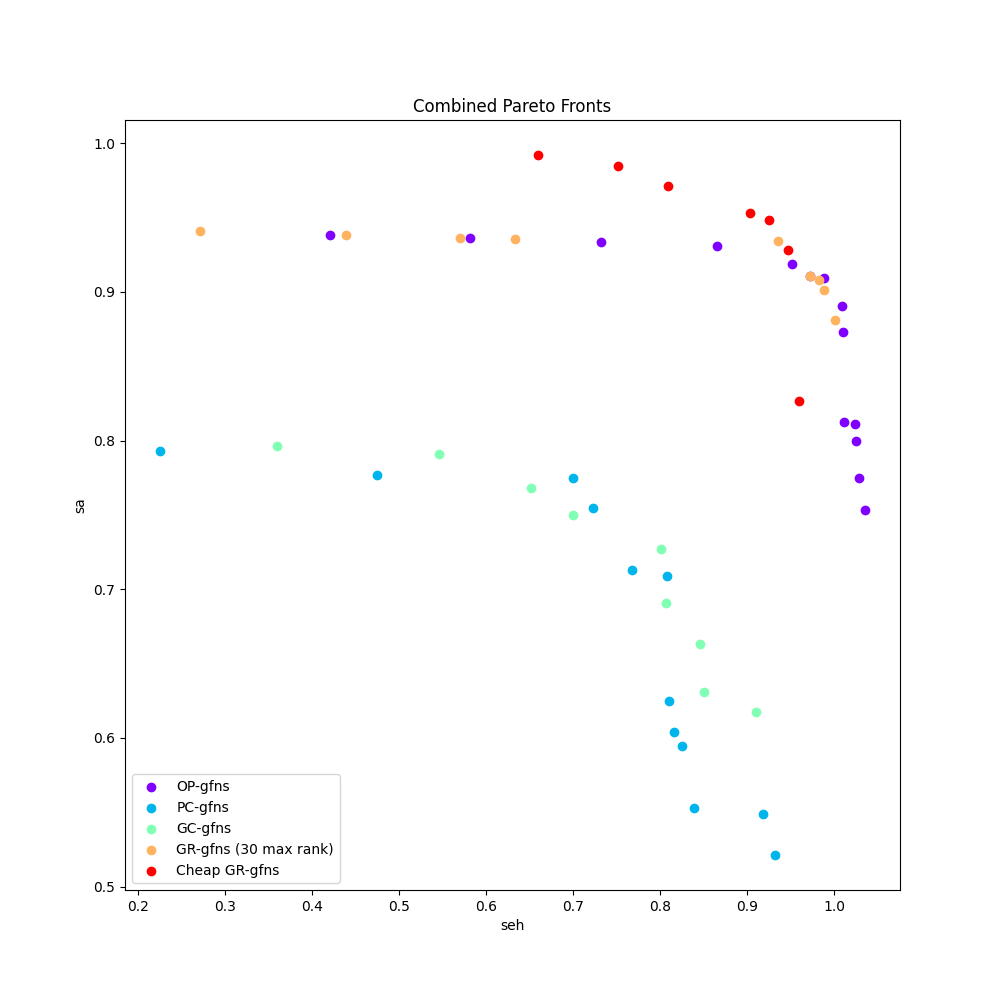}
        \label{fig:uniform}
    \end{subfigure}
     \begin{subfigure}[b]{0.49\textwidth}
        \includegraphics[width=\textwidth]{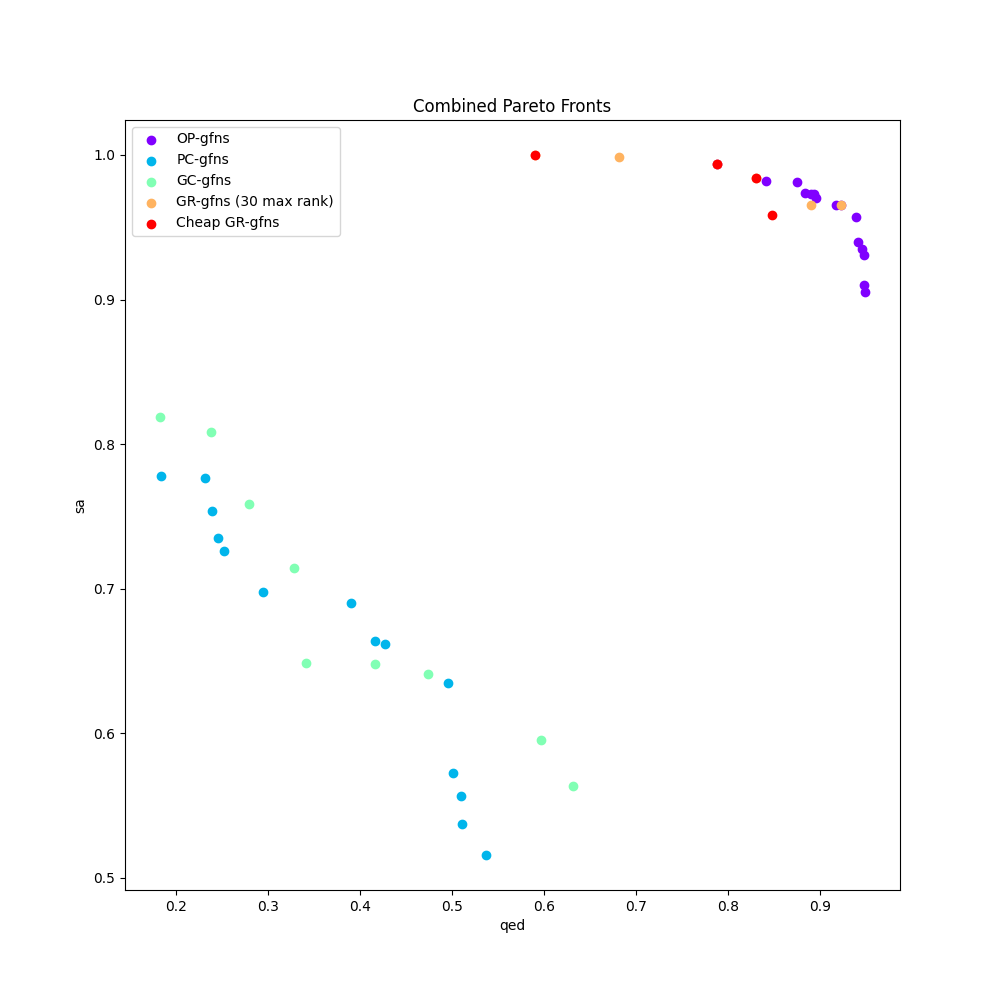}
        \label{fig:uniform}
    \end{subfigure}
    \hfill 
    \begin{subfigure}[b]{0.49\textwidth}
        \includegraphics[width=\textwidth]{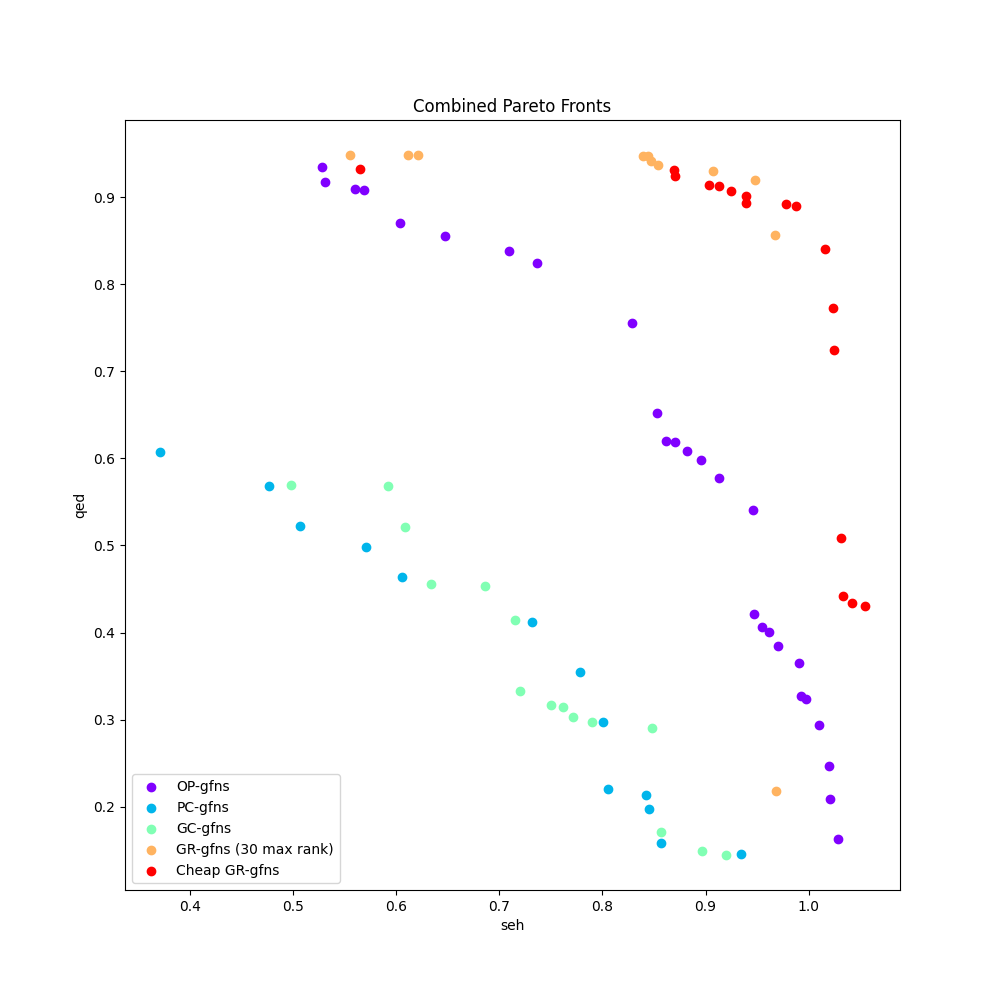}
        \label{fig:uniform}
    \end{subfigure}
    \begin{subfigure}[b]{0.49\textwidth}
        \includegraphics[width=\textwidth]{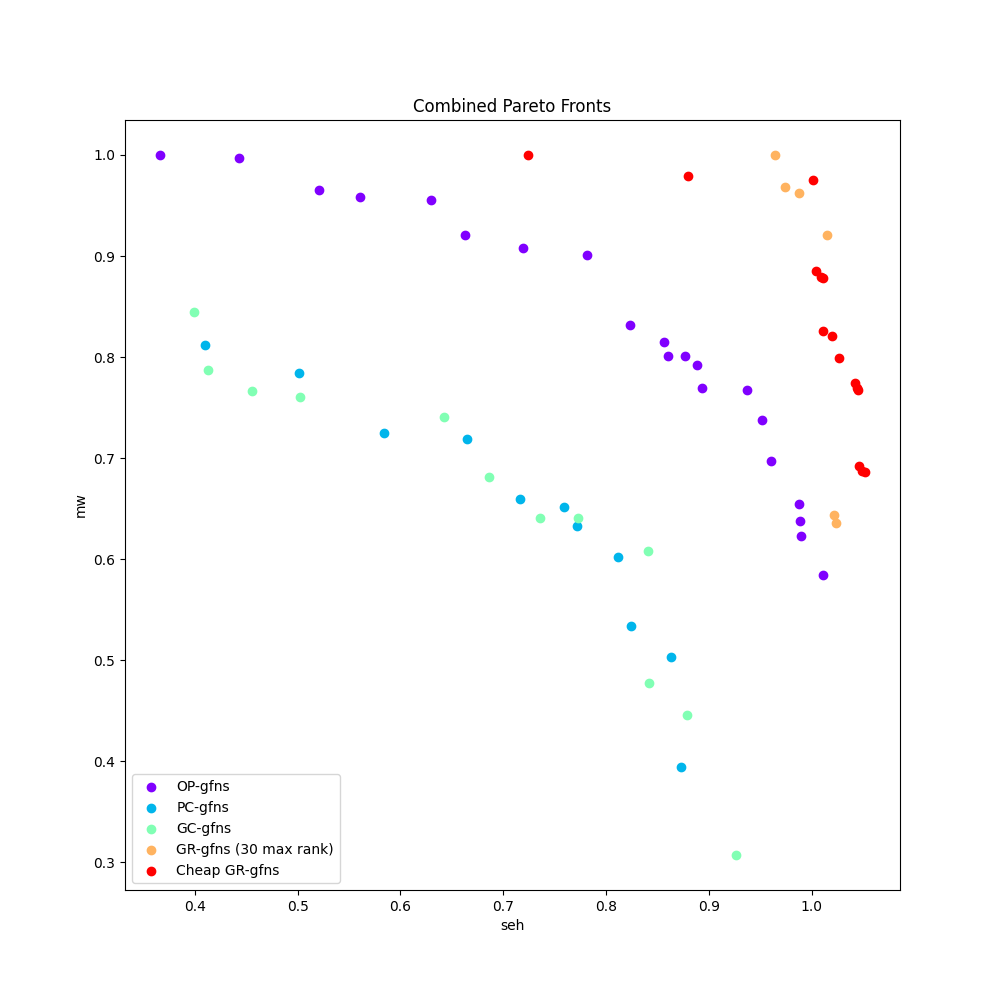}
        \label{fig:uniform}
    \end{subfigure}
    \hfill 
    \centering
     \begin{subfigure}[b]{0.49\textwidth}
        \includegraphics[width=\textwidth]{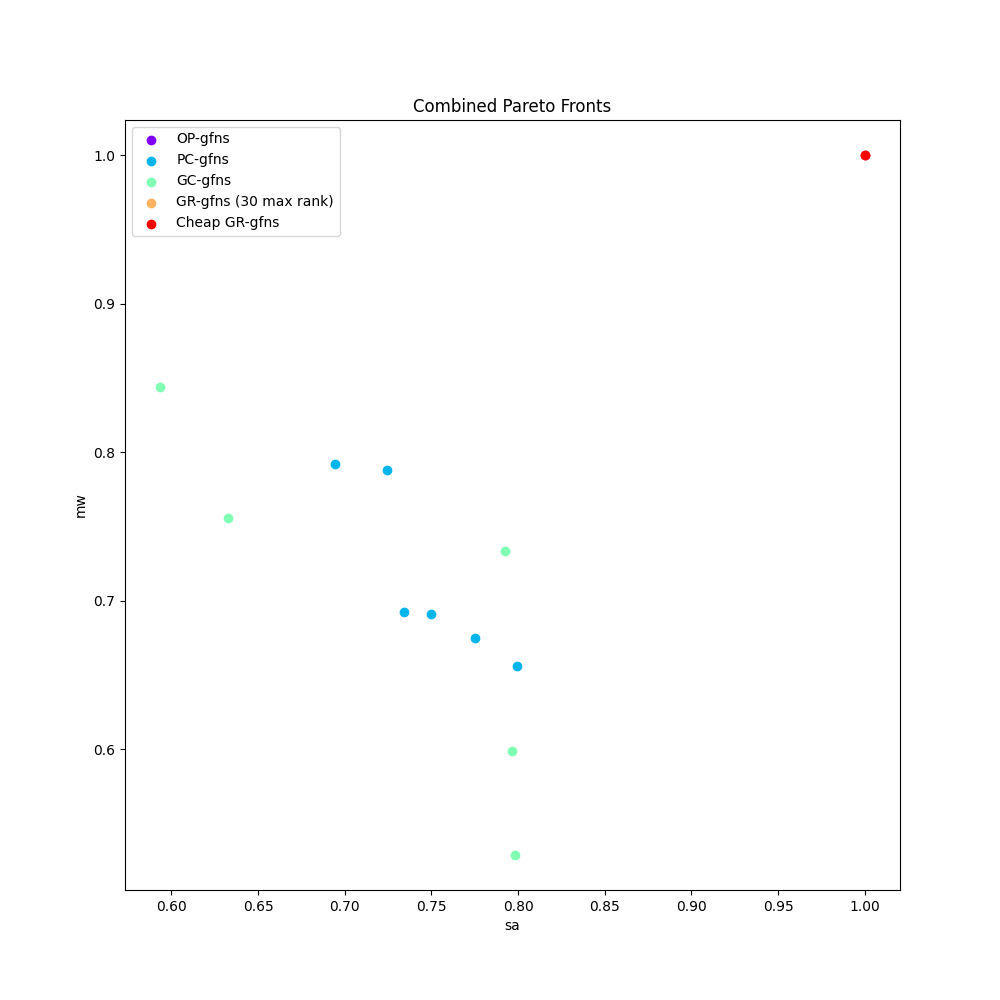}
        \label{fig:uniform}
    \end{subfigure}
    \begin{subfigure}[b]{0.49\textwidth}
        \includegraphics[width=\textwidth]{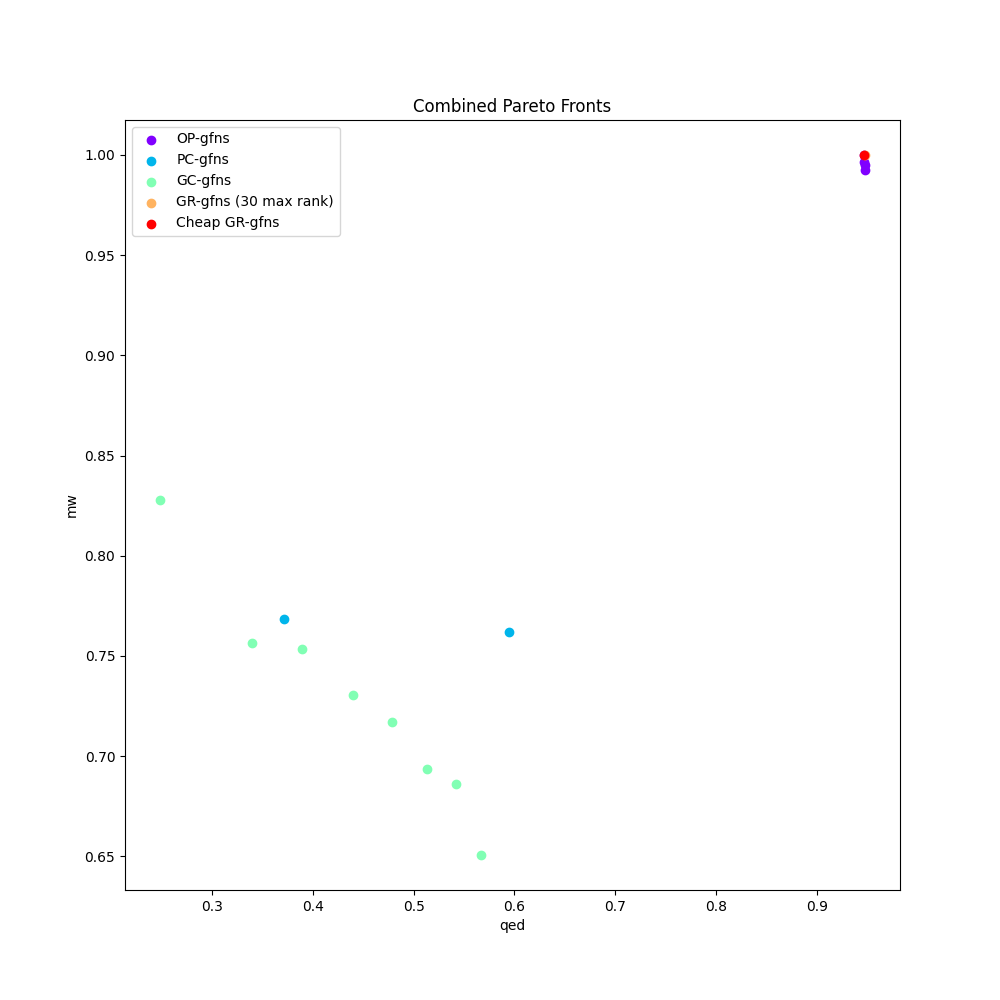}
        \label{fig:uniform}
    \end{subfigure}
\caption{Fragmentation-based molecule generation: Comparison of the different Pareto fronts provided by each method. }
\label{fig:seh_pareto}
\end{figure}
\begin{table}
\caption{Results for Fragmentation-Based Molecule Generation benchmark}
\centering
\sisetup{
  table-format=2.2,
  detect-weight=true,
  detect-inline-weight=math
}
\begin{tabular}{
  l
  S[table-format=1.2]
  S[table-format=1.2]
  S[table-format=1.2]
  S[table-format=1.2]
  S[table-format=1]
}
\toprule
{\textbf{SEH-QED}} & {PC-ent (\up)} & {IGD+ (\down)} & {d$_H(P',P)$ (\down)} & {R$_2$ (\down)} &  {Non-dominated (\up)} \\
\midrule
OP-GFNs          & \textbf{2.12} & 0.22 & \textbf{0.1819} & 7.7   & 0 \\
PC-GFNs          & 1.88 & 0.43 & 0.36   & 13.277 & 0 \\
GC-GFNs          & 1.59 & 0.44 & 0.38   & 13.35  & 0 \\
GR-GFNs (30)*     & 1.37 & 0.21 & 0.33   & \textbf{3.82} & \textbf{9} \\
Cheap GR-GFNs*    & 1.76 & \textbf{0.19} & 0.23   & 4.39   & 1 \\
\midrule
\textbf{SEH-SA} \\
\midrule
OP-GFNs          & \textbf{1.73} & \textbf{0.25} & 0.12   & \textbf{3.27}   & \textbf{9} \\
PC-GFNs          & 1.54 & 0.29 & 0.24   & 10.89  & 0 \\
GC-GFNs          & 1.58 & 0.32 & 0.28   & 10.84  & 0 \\
GR-GFNs (30)*     & 1.43 & 0.28 & 0.15   & 4.55 & 2 \\
Cheap GR-GFNs*    & 1.55 & 0.34 & \textbf{0.01}   & 3.69   & 6 \\
\midrule
\textbf{SEH-MW} \\
\midrule
OP-GFNs          & 1.95 & \textbf{0.21} & 0.18   & 6.83   & 0 \\
PC-GFNs          & 1.85 & 0.32 & 0.36   & 11.54  & 0 \\
GC-GFNs          & \textbf{2.02} & 0.29 & 0.35   & 11.27  & 0 \\
GR-GFNs (30)*     & 0.64 & 0.37 & 0.35   & 1.66   & 2 \\
Cheap GR-GFNs*    & 1.55 & 0.28 & \textbf{0.13}   & \textbf{1.35}   & \textbf{9} \\
\midrule
\textbf{QED-SA} \\
\midrule
OP-GFNs          & 0.85 & 0.39 & 0.44   & \textbf{1.93}   & \textbf{15} \\
PC-GFNs          & \textbf{1.73} & \textbf{0.06} & 0.35   & 11.64  & 0 \\
GC-GFNs          & 1.47 & 0.07 & 0.33   & 11.8   & 0 \\
GR-GFNs (30)*     & 1.33 & 0.36 & 0.24   & 3.32   & 5 \\
Cheap GR-GFNs*    & 1.04 & 0.37 & \textbf{0}      & 7.8    & 3 \\
\midrule
\textbf{QED-MW} \\
\midrule
OP-GFNs          & 0    & 0.5  & 0.83   & \textbf{1.27}   & \textbf{2} \\
PC-GFNs          & 0.69 & 0.44 & 0.38   & 12.1& 0 \\
GC-GFNs          & \textbf{1.32} & \textbf{0.43} & \textbf{0.36}   & 12.9   & 0 \\
GR-GFNs (30)*     & 0    & 0.51 & 0.81   & 1.72   & 1 \\
Cheap GR-GFNs*    & 0    & 0.51 & 0.43   & 2.39   & 0 \\
\midrule
\textbf{SA-MW} \\
\midrule
OP-GFNs          & 0    & 0.53 & \textbf{0}      & \textbf{0}      & \textbf{1} \\
PC-GFNs          & 0.87 & 0.38 & 0.32   & 10.86  & 0 \\
GC-GFNs          & \textbf{1.33} & \textbf{0.32} & 0.27   & 10.19  & 0 \\
GR-GFNs (30)*     & 0    & 0.53 & \textbf{0}      & \textbf{0}      & \textbf{1} \\
Cheap GR-GFNs*    & 0    & 0.53 & 0.02   & \textbf{0}      & \textbf{1} \\
\bottomrule
\end{tabular}

\label{tab:seh_table}
\end{table}

\newpage
\subsection{QM9}\label{appendix:QM9}
\subsubsection{Objective Functions}
The environment QM9 \cite{qm9} provides different metrics when evaluating the sequentally generation of molecules of up to 9 atoms and different bonds.
Following previous works \cite{moo_gfn,OP-Gfns} we consider the following objective functions:
\begin{itemize}
    \item \texttt{MXMNet}: This is the main reward, a proxy \cite{mxmnet} trained to predict the HOMO-LUMO gap.
    \item \texttt{logP}: The molecular logP target which can be extracted from RDKit: $\text{R}_{logP} = \exp\left(-(\texttt{logP} - 2.5)^2 / 2\right)$
    \item \texttt{SA}: Same as \ref{func:SA}
    \item \texttt{MW}: In this case we modify this function following what has been done in \cite{moo_gfn}: $\text{R}_{MW} = \exp\left(-(\texttt{MW} - 105)^2 /150 \right)$
\end{itemize}
\subsubsection{Network Structure and Training Details}
We examined them under varying training durations. Specifically, we subjected both PC-GFNs and GR-GFNs to 100000 steps. In contrast, OP-GFNs underwent a shorter training period with 50000 steps, while Cheap-GR-GFNs had the least, with 30000 steps. As in the previous experiment the chosen model is the GNN, with 4 layers and number of embeddings being 128.
\subsubsection{Plots and Tables}
\begin{figure}[h]
    \centering
    \includegraphics[width=1\textwidth]{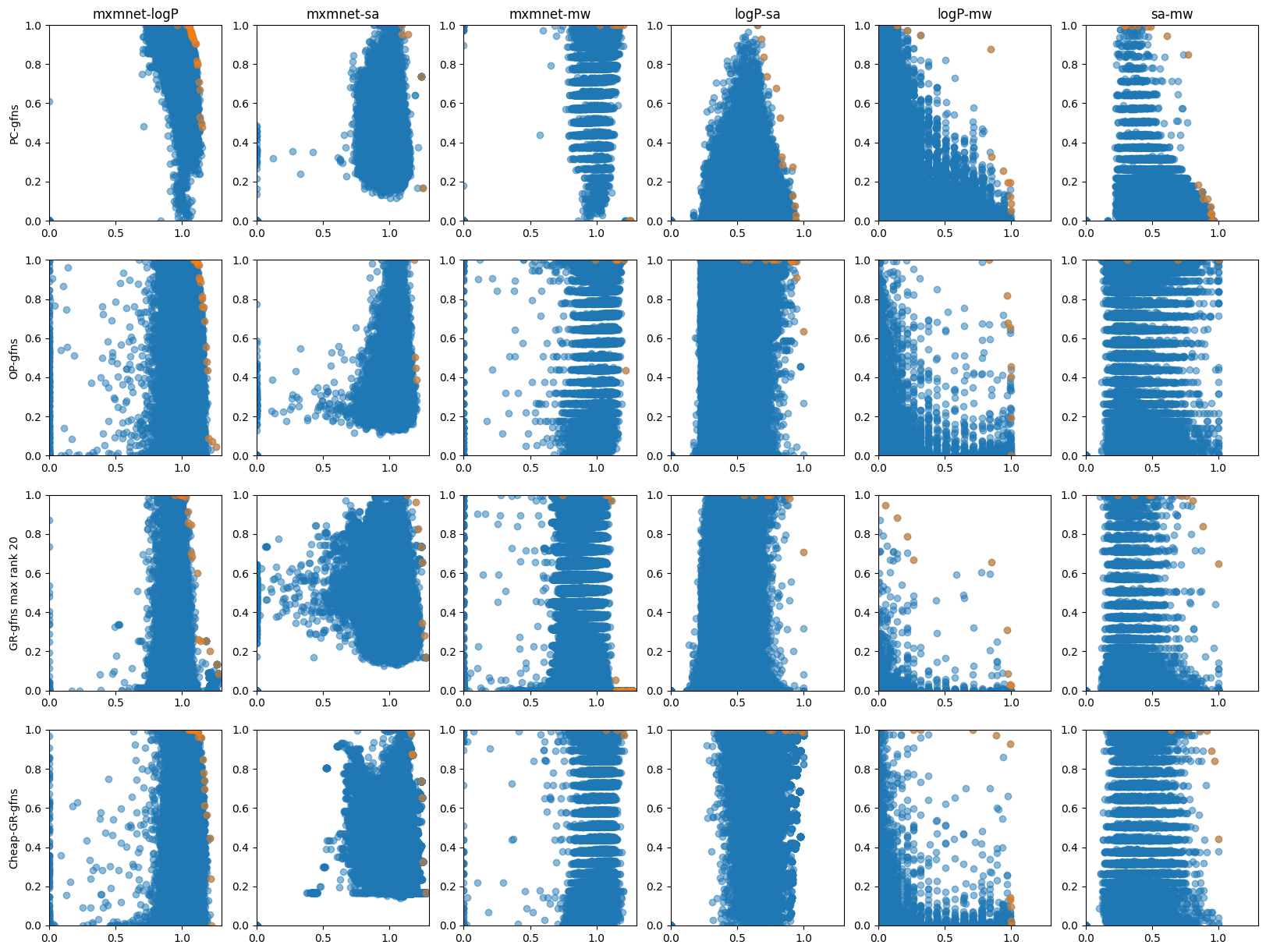}
    \caption{All generated samples (blue) for each method and their respective Pareto front (orange)}
    \label{fig:qm9_plots}
\end{figure}
\begin{table}[h]
\caption{Results for QM9 benchmark}
\centering
\sisetup{
  table-format=2.2,
  detect-weight=true,
  detect-inline-weight=math
}
\begin{tabular}{
  l
  S[table-format=1.2]
  S[table-format=1.2]
  S[table-format=1.2]
  S[table-format=2.2]
  S[table-format=1.0]
}
\toprule
{\textbf{MXMNet-logP}} & {PC-ent (\up)} & {IGD (\down)} & {d$_H(P',P)$ (\down)} & {R$_2$ (\down)} &  {Non-dominated (\up)} \\
\midrule
PC-GFNs          & 1.38 & 0.33 & \textbf{0.1}   & 62.47 & 0 \\
OP-GFNs        & 1.55 & 0.31 & 0.71  & 44.8  & 7 \\
GR-GFNs (20)* & 1.37 & \textbf{0.29} & 1.01  & \textbf{8.7}   & 2 \\
Cheap GR-GFNs*   & \textbf{1.75} & \textbf{0.29} & 0.88  & 26.11 & \textbf{12} \\
\midrule
\textbf{MXMNet-SA} \\
\midrule
PC-GFNs           & 0.95 & 0.33 & 0.33  & 66.25 & 0 \\
OP-GFNs        & 1.39 & 0.39 & 0.77  & 33.32 & 1 \\
GR-GFNs (20)* & 1.56 & \textbf{0.31} & 0.69  & 37.23 & \textbf{6} \\
Cheap GR-GFNs*   & \textbf{1.67} & 0.32 & \textbf{0.24}  & \textbf{32.79} & 3 \\
\midrule
\textbf{MXMNet-MW} \\
\midrule
PC-GFNs           & \textbf{0.96} & 0.35 & \textbf{0.09}  & 32.2  & 3 \\
OP-GFNs        & 0.87 & 0.32 & 0.95  & \textbf{19.64} & 3 \\
GR-GFNs (20)* & 0.82 & \textbf{0.27} & 0.42  & 27.37 & \textbf{5} \\
Cheap GR-GFNs*   & 0.56 & 0.5  & 1.09  & 25.01 & 3 \\
\midrule
\textbf{logP-SA}\\
\midrule
PC-GFNs           & \textbf{1.82} & 0.23 & \textbf{0.41}  & 51.74 & 0 \\
OP-GFNs        & 1.63 & \textbf{0.22} & 0.82  & 8.33  & 1 \\
GR-GFNs (20)* & 1.49 & 0.25 & 1.1   & 8.82  & 0 \\
Cheap GR-GFNs*   & 1.04 & 0.41 & 0.44  & \textbf{3.54}  & \textbf{8} \\
\midrule
\textbf{logP-MW} \\
\midrule
PC-GFNs           & \textbf{1.98} & 0.14 & 0.77  & 35.18 & 2 \\
OP-GFNs        & 1.73 & 0.22 & 0.99  & 1.43  & \textbf{5} \\
GR-GFNs (20)* & 1.89 & 0.2  & \textbf{0.72}  & 1.66  & 1 \\
Cheap GR-GFNs*   & 1.83 & \textbf{0.13} & 0.99  & \textbf{0.62}  & 3 \\
\midrule
\textbf{SA-MW} \\
\midrule
PC-GFNs           & 1.75 & \textbf{0.17} & \textbf{0.36}  & 51.64 & 0 \\
OP-GFNs        & 1.1  & 0.32 & 0.83  & \textbf{0.01}  & 1 \\
GR-GFNs (20)* & \textbf{1.89} & 0.18 & 1.02  & 6.6   & 0 \\
Cheap GR-GFNs*   & 1.56 & 0.2  & 0.98  & 7.04  & \textbf{3} \\
\bottomrule
\end{tabular}

\label{tab:qm9_table}
\end{table}
\begin{figure}
    \centering
    \begin{subfigure}[b]{0.49\textwidth}
        \includegraphics[width=\textwidth]{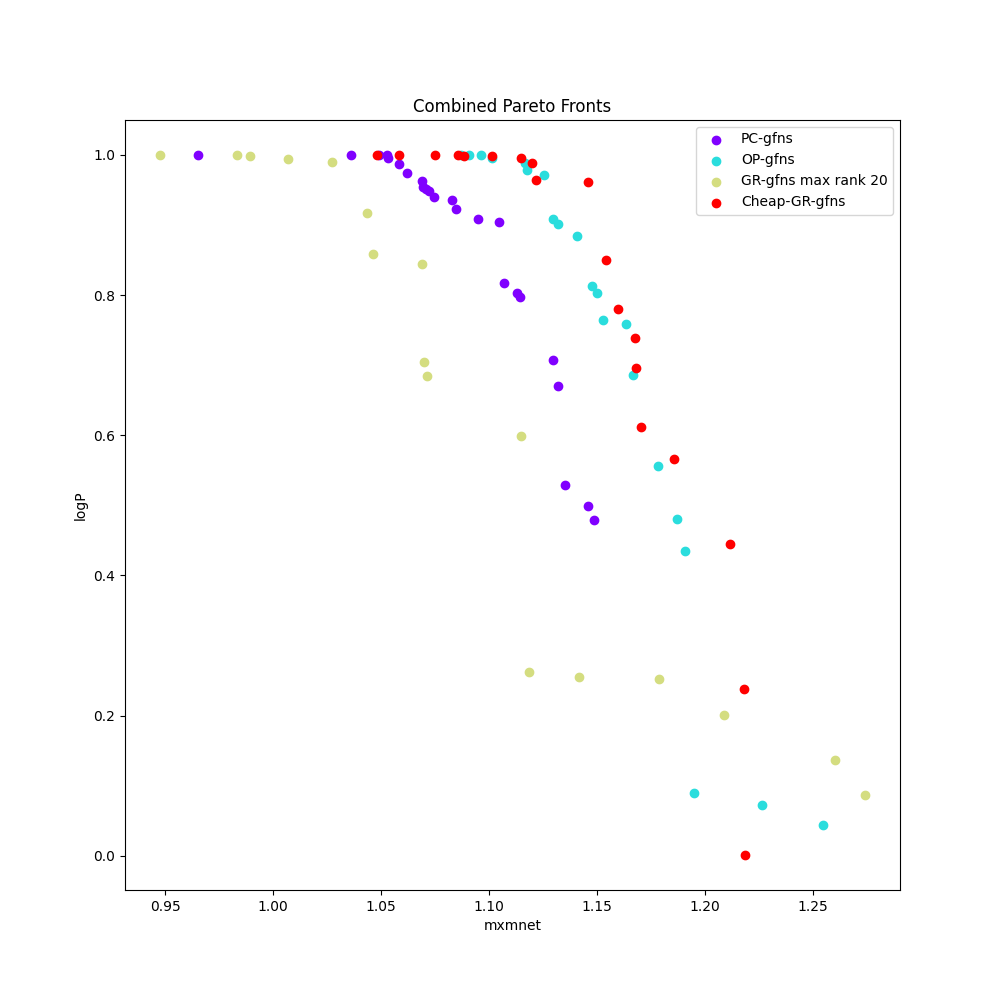}
        \label{fig:uniform}
    \end{subfigure}
     \begin{subfigure}[b]{0.49\textwidth}
        \includegraphics[width=\textwidth]{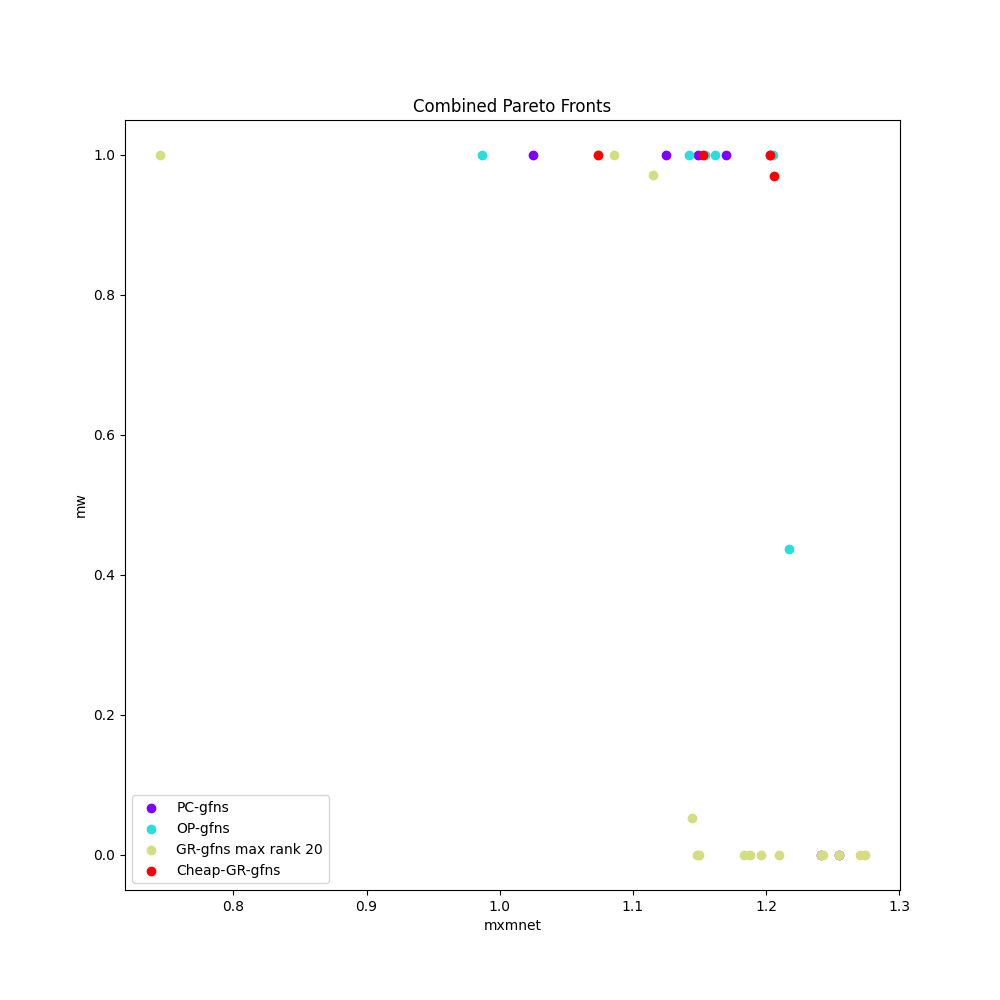}
        \label{fig:uniform}
    \end{subfigure}
    \hfill 
    \begin{subfigure}[b]{0.49\textwidth}
        \includegraphics[width=\textwidth]{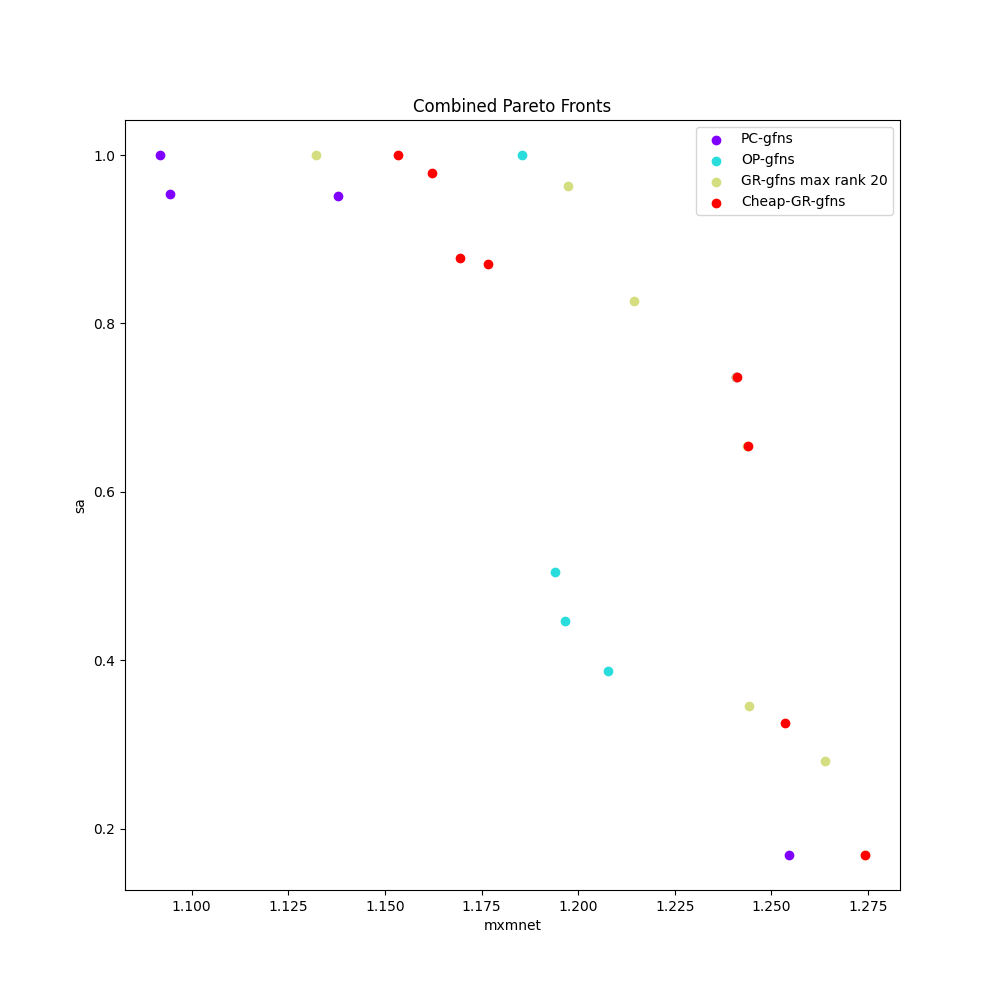}
        \label{fig:uniform}
    \end{subfigure}
    \begin{subfigure}[b]{0.49\textwidth}
        \includegraphics[width=\textwidth]{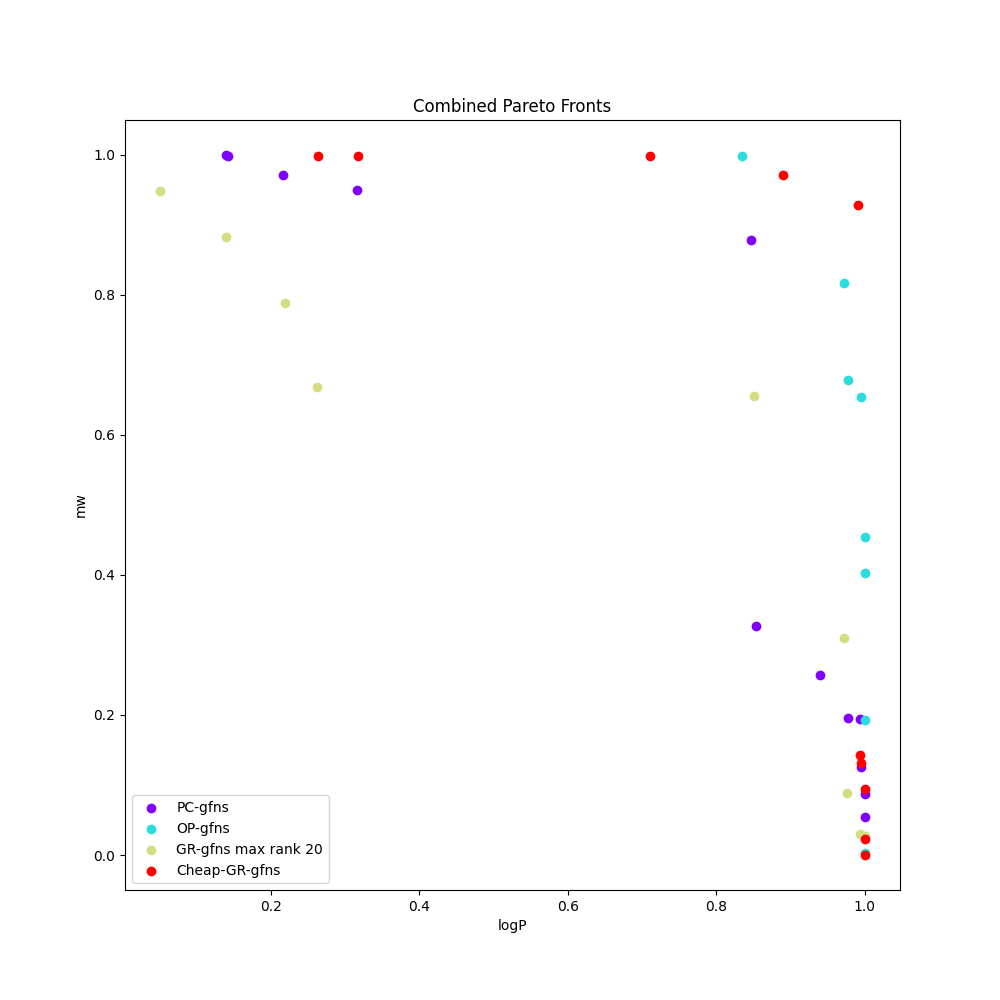}
        \label{fig:uniform}
    \end{subfigure}
    \hfill 
    \centering
     \begin{subfigure}[b]{0.49\textwidth}
        \includegraphics[width=\textwidth]{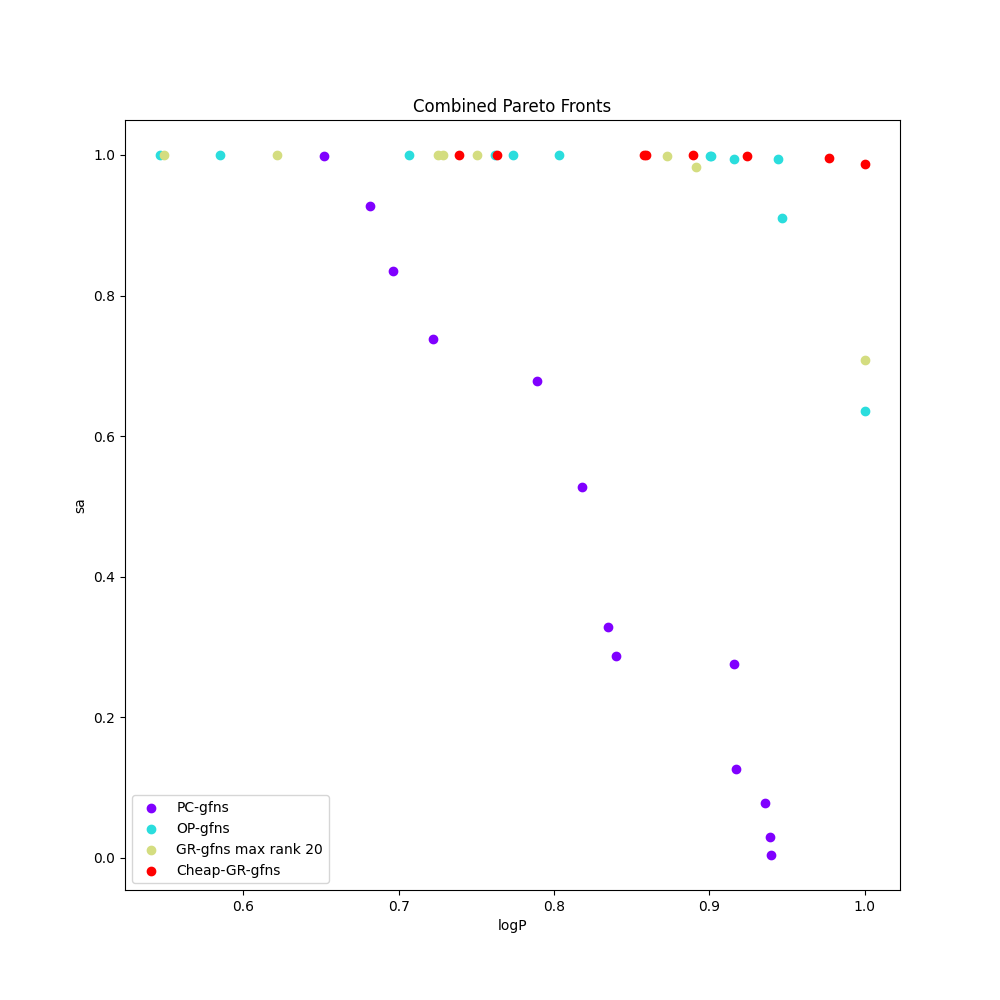}
        \label{fig:uniform}
    \end{subfigure}
    \begin{subfigure}[b]{0.49\textwidth}
        \includegraphics[width=\textwidth]{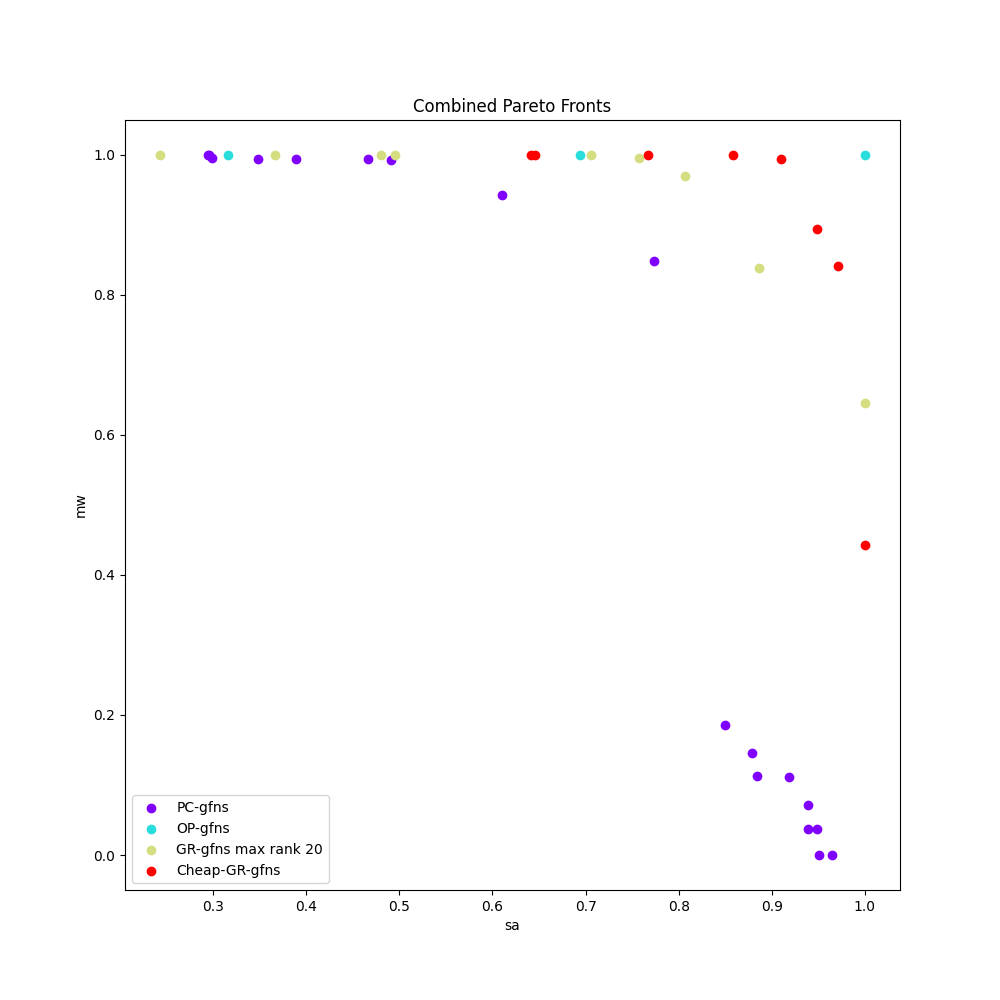}
        \label{fig:uniform}
    \end{subfigure}
\caption{QM9: Comparison of the different Pareto fronts provided by each method.}
\label{fig:qm9_pareto_plots}
\end{figure}

\end{document}